\documentclass[10pt,journal,cspaper,compsoc]{IEEEtran}
\ifCLASSOPTIONcompsoc
  \usepackage[nocompress]{cite}
\else
   \usepackage{cite}
\fi
\ifCLASSINFOpdf
  \usepackage[pdftex]{graphicx}
  \graphicspath{{./}}
  \DeclareGraphicsExtensions{.pdf,.jpg,.jpeg,.png}
\else
  \usepackage[dvips]{graphicx}
   \graphicspath{{./}}
  \DeclareGraphicsExtensions{.eps}
\fi
\usepackage[cmex10]{amsmath}
\usepackage{booktabs}
\usepackage[table]{xcolor}
\usepackage{color, colortbl}
\usepackage{tikz}
\usetikzlibrary{shapes,arrows}
\usepackage{multirow}
\usepackage{flafter}
\usepackage{float}

\usepackage{algorithmic}
\usepackage[ruled,lined,commentsnumbered]{algorithm2e}

\usepackage{array}

\usepackage{mdwmath}
\usepackage{mdwtab}

\usepackage{eqparbox}

\ifCLASSOPTIONcompsoc
\usepackage[tight,normalsize,sf,SF]{subfigure}
\else
\usepackage[tight,footnotesize]{subfigure}
\fi

\usepackage{stfloats}
\usepackage{hyperref}

\usepackage{breqn}
\newtheorem{theorem}{Theorem}
\newtheorem{definition}{Definition}
\DeclareMathOperator*{\argmin}{argmin} 

\begin{document}
%
\title{A Framework for Supervised Heterogeneous Transfer Learning using Dynamic Distribution Adaptation and Manifold Regularization}
%
%
%
%

\author{Md~Geaur~Rahman and~Md~Zahidul~Islam 
\IEEEcompsocitemizethanks{\IEEEcompsocthanksitem Md~Geaur~Rahman is a Research Fellow in the School of Computing, Mathematics and Engineering, Charles Sturt University, Australia, and a Professor in the Department of Computer Science and Mathematics, Bangladesh Agricultural University, Bangladesh.\protect\\
E-mail: grahman@csu.edu.au
\IEEEcompsocthanksitem Md~Zahidul~Islam is a Professor in Computer Science, in the School of Computing, Mathematics and Engineering, Charles Sturt University, Bathurst, NSW 2795, Australia.\protect\\
E-mail: zislam@csu.edu.au}
\thanks{Revised manuscript received August~30,~2022}}

%
%

\markboth{IEEE Transactions on Services Computing, August~2022}%
{Rahman \MakeLowercase{\textit{et al.}}: TLF: A Transfer Learning Framework}
%

\IEEEpubid{0000--0000/00\$00.00~\copyright~2022 IEEE}

\IEEEcompsoctitleabstractindextext{%
\begin{abstract}

Transfer learning aims to learn classifiers for a target domain by transferring knowledge from a source domain. However, due to two main issues: feature discrepancy and distribution divergence, transfer learning can be a very difficult problem in practice. In this paper, we present a framework called TLF that builds a classifier for the target domain having only a few labeled training records by transferring knowledge from the source domain having many labeled records. While existing methods often focus on one issue and leave the other one for further work, TLF is capable of handling both issues simultaneously. In TLF, we alleviate feature discrepancy by identifying shared label distributions that act as the pivots to bridge the domains. We handle distribution divergence by simultaneously optimizing the structural risk functional, joint distributions between domains, and the manifold consistency underlying marginal distributions. Moreover, for the manifold consistency we exploit its intrinsic properties by identifying $k$ nearest neighbors of a record, where $k$ is determined automatically. We evaluate TLF on seven publicly available datasets and compare performances of TLF and fourteen state-of-the-art techniques. Our experimental results, including statistical sign test and Nemenyi test analyses, indicate a clear superiority of TLF over the state-of-the-art techniques.

\end{abstract}

}

\maketitle

\IEEEdisplaynotcompsoctitleabstractindextext

%
\IEEEpeerreviewmaketitle

\section{Introduction}
\label{introduction}

Nowadays, Transfer Learning (TL) has made breakthroughs in solving real-world problems, including activity recognition, sentiment classification, document analysis and indoor localization~\cite{niu2020decade, sukhija2019supervised}. Traditional Machine Learning (ML) methods cannot handle these problems due to various reasons, including insufficient data for building a model and distribution mismatch between training and test datasets~\cite{long2014adaptation, wang2020transfer}.

TL algorithms build a model for a domain (often referred to as the target domain), which has insufficient data, by transferring knowledge from single or multiple auxiliary domains (referred to as source domains)~\cite{niu2020decade, sukhija2019supervised}. However, commonly used TL methods assume that the source domain dataset (SDD) and target domain dataset (TDD) have the same set of attributes~\cite{long2014adaptation, bruzzone2009domain}. In addition, traditional supervised TL methods assume that the datasets have the same set of labels~\cite{long2014adaptation, wang2020transfer, wang2019easytl}. The methods may not perform well if the datasets have different set of attributes and labels~\cite{sukhija2019supervised}. The domains of such datasets often known as heterogeneous domains~\cite{sukhija2019supervised}. 
\begin{figure*}[ht!]
\centering
  \setlength{\belowcaptionskip}{0pt}
	\setlength{\abovecaptionskip}{0pt}	
	    \subfigure[Source domain: A toy dataset (Product) of a hypothetical multinational company ``ABC'']
	    {
	        \includegraphics[width=0.50\linewidth]{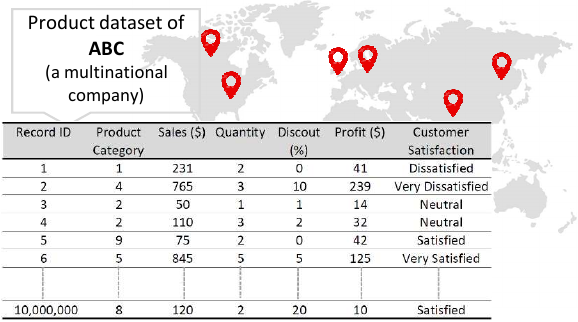}
	        \label{fig:source_toy}
	    }
	    \subfigure[Target domain: A toy dataset (Product) of a hypothetical Australian company ``XYZ'']
	    {
	        \includegraphics[width=0.45\linewidth]{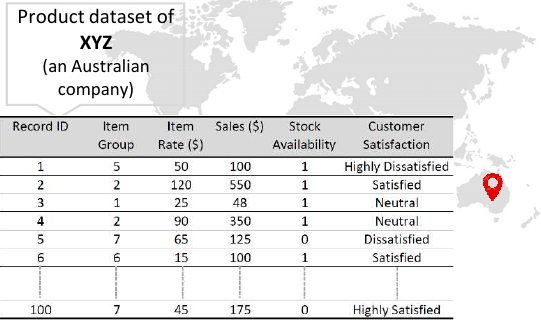}
	        \label{fig:target_toy}
	    }
	    \caption{Sample source and target domains using toy datasets.}
	    \label{fig:toy_domains}
\end{figure*}

A hypothetical heterogeneous setting is illustrated in Fig.~\ref{fig:toy_domains}, where Fig.~\ref{fig:toy_domains}a and Fig.~\ref{fig:toy_domains}b show two toy datasets related to products, sales and customer feedback information of a renowned multinational electronics company ``ABC'' and a newly established Australian superstore ``XYZ'', respectively. The datasets are heterogeneous in terms of attributes and labels. The ``XYZ'' company aims to evaluate the products based on the customer feedback. Since the dataset of ``XYZ'' company contains only few labeled records, it may not be possible to build a good model by applying an ML algorithm on the dataset. Besides, a good model for the ``XYZ'' company can be built by transferring records in a compatible way from the dataset of ``ABC'' company which has many records. However, traditional TL methods, such as ARTL~\cite{long2014adaptation} and MEDA~\cite{wang2020transfer} are not capable of transferring knowledge from source to target domain since the datasets have different sets of features and labels.   
\begin{figure*}[ht!]
\centering
  \setlength{\belowcaptionskip}{0pt}
	\setlength{\abovecaptionskip}{0pt}	
	\includegraphics[width=0.98\linewidth]{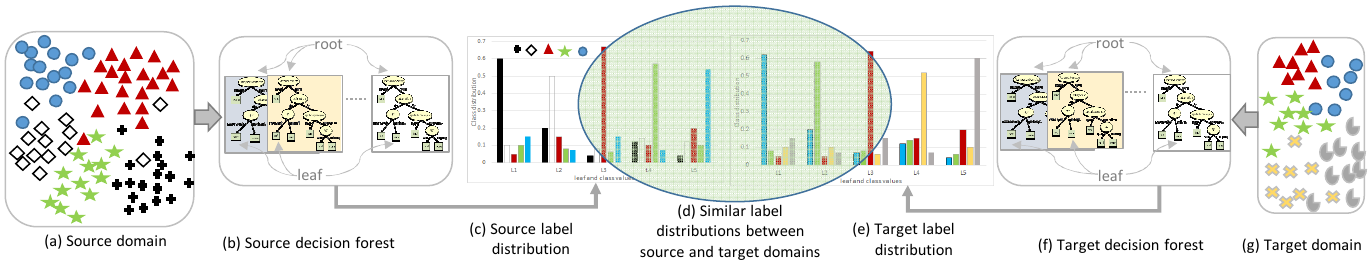}
	 \caption{Basic concept of bridging between source and target domains.}
	 \label{fig:basicconcept}
\end{figure*}

For dealing with such settings, a number of heterogeneous transfer learning (HeTL) methods have been proposed, and have recently received significant attention in the literature~\cite{fang2022semi, zhen2019deep, sukhija2019supervised, yao2019heterogeneous}. The methods either learn a domain-invariant space~\cite{zhen2019deep, ren2019heterogeneous} for representing the data of $D_s$ and $D_t$ or learn a mapping to project such data~\cite{fang2022semi, yao2019heterogeneous} or explore correspondences to bridge the heterogeneity gap between $D_s$ and $D_t$~\cite{sukhija2019supervised, chen2016transfer}. For clear understanding about the process of exploring the correspondences, we briefly discuss an existing method called SHDA~\cite{sukhija2019supervised} as follows. 

SHDA~\cite{sukhija2019supervised} deals with the HeTL setting by exploring correspondences and then transferring records of SDD to TDD. SHDA projects source records for TDD by identifying similar label distributions that act as the bridge between SDD and TDD. The concept of bridging the domains is illustrated in Fig~\ref{fig:basicconcept}. Using the SDD (Fig.~\ref{fig:basicconcept}a), SHDA builds a decision forest (in short, forest) which consists of a set of decision trees (in short, trees) as shown in Fig.~\ref{fig:basicconcept}b. A tree has a set of rules, where a rule is used to create a leaf containing records that satisfy the conditions of the rule. Using the records of a leaf, SHDA calculates label distributions of the leaf. The hypothetical label distributions of all leaves for the source domain are shown in Fig.~\ref{fig:basicconcept}c. Similarly, for the TDD (Fig.~\ref{fig:basicconcept}g), SHDA builds a forest (Fig.~\ref{fig:basicconcept}f) and then calculates the label distributions of all leaves (Fig.~\ref{fig:basicconcept}e). After that SHDA finds similar label distributions (see Fig.~\ref{fig:basicconcept}d) between the domains. The similar label distributions act as the pivots to bridge the domains and help achieving accurate transfer of data.

For both source and target forests, SHDA then calculates attribute contributions of the leaves that are associated with the pivots. The contribution of an attribute $A_j$ is calculated as $(\frac{1}{2})^l$, where $l$ is the level of the node where $A_j$ is contributed. As an example, for Leaf 3 of the sample tree (see Fig.~\ref{fig:sampletree}) the contribution of attributes ``Item Group'' and ``Sales'' are  0.625 (=0.500+0.125) and 0.250 respectively.
\begin{figure}[ht!]
\centering
  \setlength{\belowcaptionskip}{0pt}
	\setlength{\abovecaptionskip}{0pt}	
	\includegraphics[width=0.7\linewidth]{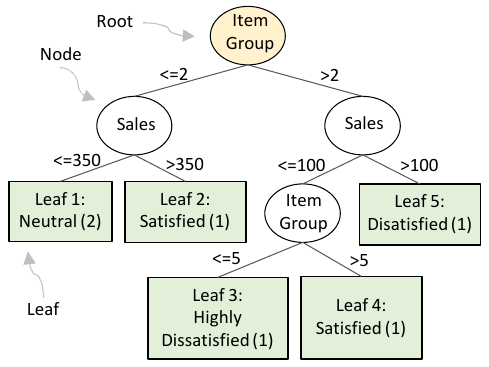}
	 \caption{A sample decision tree.}
	 \label{fig:sampletree}
\end{figure}

The attribute contributions are then used in Eq.~\eqref{shda_ps} to project the records of SDD into TDD. Combining the projected records with the TDD, SHDA builds the final model as the target model which is then used to classify the test data (discussed in Section~\ref{impact_targetdata}).

In HeTL, a good method should address two challenging issues~\cite{pan2010survey}: distribution divergence (see Definition~\ref{def_distributiondivergence}) and feature discrepancy (see Definition~\ref{def_featurediscrepancy}). However, existing methods often focus on one issue and leave the other one for the further work. For example, ARTL~\cite{long2014adaptation} and MEDA~\cite{wang2020transfer} focus on minimizing distribution divergence while SHDA~\cite{sukhija2019supervised} attends to learning shared label distributions.

ARTL and MEDA minimize the conditional distribution divergence by using maximum mean discrepancy (MMD)~\cite{pan2008transfer, liu2020learning} and marginal distribution divergence by exploiting intrinsic properties of manifold consistency~\cite{belkin2006manifold}. However, to identify the nearest neighbors of a record during manifold regularization, the methods use a user-defined parameter $k$ which may fail to exploit the properties accurately (explained later in Section~\ref{generalframework} and Fig.~\ref{fig:manifold_k}) resulting in a low classification accuracy. Moreover, the methods are unable to handle the feature discrepancy. 

The feature discrepancy is handled in SHDA~\cite{sukhija2019supervised} by identifying pivots (discussed earlier) between SDD and TDD, the distribution divergence, however, is not addressed. Moreover, it provides higher weights to parent nodes than their children while calculating attribute contributions. For a leaf, the importance of a test attribute in classifying the records belonging to the leaf does not necessarily depend on the position/level of the attribute in the rule. It also transfers all records from SDD to TDD and may result in negative transfer~\cite{pan2010survey}.

From the above discussion, we have the following observation. The classification accuracy of a method can be improved, if: (1) the issues related to distribution divergence feature discrepancy are addressed simultaneously, (2) the nearest neighbors in manifold regularization is determined automatically, (3) the centroid of a pivot (instead of the level of a node of a tree) is used in calculating attribute contributions, and (4) the negative transfer is reduced. 

\begin{figure*}[ht!]
\centering
  \setlength{\belowcaptionskip}{0pt}
	\setlength{\abovecaptionskip}{0pt}	
	\includegraphics[width=0.98\linewidth]{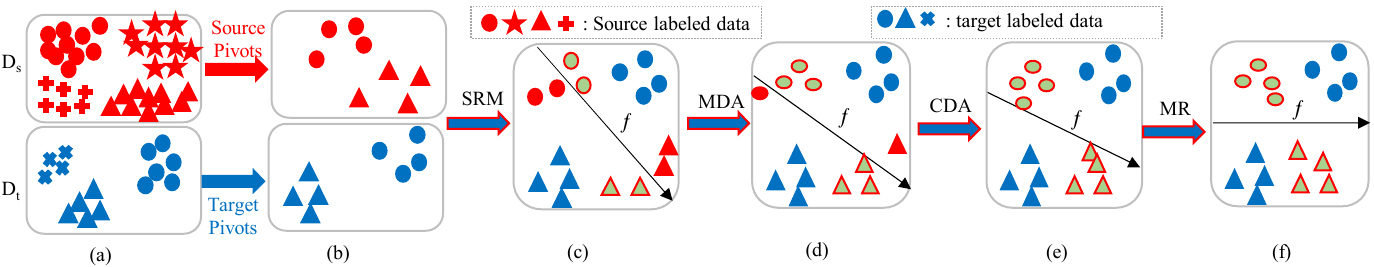}
	 \caption{Motivation of TLF where \textit{f} denotes hyperplane. (a) Original domains. (b) Source and target pivots. (c) After structural risk minimization (SRM). (d) After marginal distribution adaptation (MDA). (e) After conditional distribution adaptation (CDA). (f) After manifold regularization (MR).}
	 \label{fig:motivation}
\end{figure*}
The existing methods therefore have room for further improvement. From the above observation and discussion, we summarize the issues of TL and highlight our motivation in Fig.~\ref{fig:motivation}. Fig.~\ref{fig:motivation}a shows a source domain $D_s$ having huge labeled data and a target domain $D_t$ having a few labeled data. To minimize feature discrepancy, we build two forests from $D_s$ and $D_t$, separately and identify source and target pivots (discussed in Section~\ref{generalframework}) to bridge the gap between $D_s$ and $D_t$. Fig.~\ref{fig:motivation}b shows the centroids of the pivots which are combined into a single dataset $D_c$. Since our goal is to map the source data into target domain, we learn a projection matrix based on the target data. Using the pivots of $D_t$, we build a initial model (\textit{f}) by minimizing the structural risk function. Fig.~\ref{fig:motivation}c shows that the model cannot discriminate the pivots of $D_s$ correctly due to substantial distribution difference. Similar to ARTL and MEDA, we minimize the marginal distribution divergence (see Fig.~\ref{fig:motivation}d) and the conditional distribution divergence (see Fig.~\ref{fig:motivation}e) by using the MMD. It is worth mentioning that the importance of marginal distribution adaptation (MDA) and conditional distribution adaptation (CDA) is determined through an adaptive factor which is calculated dynamically (see Eq.~\eqref{eq_joint_mmd}). After MDA and CDA, the model can classify the pivots of $D_s$ more correctly. Finally, we maximize the manifold consistency by automatically exploiting its intrinsic properties (see Eq.~\eqref{eq_manifold1}) and obtain the perfect model as shown in Fig.~\ref{fig:motivation}f. Inspired by the above demonstration, we design a general framework by integrating all these learning objectives.

In this paper, we present a novel transfer learning framework called TLF that builds a classifier for TDD by transferring knowledge from SDD. TLF is capable of handling the two issues. The main contributions of TLF are summarized as follows.
\begin{itemize}
	\item We eliminate the feature discrepancy by identifying pivots between SDD and TDD, and then mapping the pivots into TDD. We argue that since the test attributes of a rule are equally important to classify the records belonging to a leaf (see Fig.~\ref{fig:sampletree}), it is sensible to use the leaf's centroid (see Eq.~\eqref{eq_centroid}) for representing the leaf with actual values instead of the level of the test attributes. 
	\item We minimize distribution divergence by using MMD and manifold simultaneously where the intrinsic property of the manifold consistency is exploited automatically (discussed in Section~\ref{generalframework}).
	\item We minimize negative transfer, during the projection, by considering only the source records that are belonging to the source pivots (discussed in Section~\ref{generalframework}).
	\item Since TLF produces trees (see Fig.~\ref{fig:sampletree}) as the final classifier, it can be useful for knowledge discovery purposes.
	\item We explore the necessity of transfer learning across domains.
	\item We explore the impact of target dataset size on TLF performances and find the minimum labeled data required for the TLF to build a good model for the target domain (see Section~\ref{impact_targetdata}) while ML methods fail to build a good model with the same dataset. 
	\item We carry out extensive experimentation to evaluate the effectiveness of TLF over the existing methods. 
\end{itemize}
We evaluate TLF on seven publicly available real datasets by comparing its performance with the performance of fourteen high quality existing techniques including two baseline algorithms. We also analyse the impact of target dataset size on the performances of TLF and identify three scenarios, namely  Minimun target dataset size (MTDS), Average target dataset size (ATDS), and Best target dataset size (BTDS). Our experimental results and statistical analyses indicate that TLF performs significantly better than the existing methods in all scenarios. 

The rest of the paper is organized as follows. Section~\ref{problemformulation} presents problem formulation and assumptions, and Section~\ref{relatedwork} presents a background study on transfer learning. Our proposed transfer learning framework is presented in Section~\ref{ourtechnique}. Section~\ref{experiments} presents experimental results, Section~\ref{Analysis_Challenging_Situations} presents the effectiveness of TLF in some challenging situations, and Section~\ref{conclusion} gives concluding remarks.
\section{Problem Formulation}
\label{problemformulation}

In this section, we define the problem setting, assumptions and goal for supervised heterogeneous transfer learning. The notations that we use frequently in this paper, are presented in Table~\ref{tab:notations}.
\begin{table}[ht!]
	\scriptsize
	\centering
	\setlength{\belowcaptionskip}{-10pt}
	\setlength{\abovecaptionskip}{-5pt}		
	\caption{Notations and their meanings}
	\renewcommand\tabcolsep{3pt} 
	\begin{tabular}{ll|ll}
	\toprule
		Notation&Meaning&Notation&Meaning\\
		\midrule
	$D_s$&Source domain&$D_t$&Target domain\\
	$n_s$&Size of $D_s$&$n_t$&Size of $D_t$\\
	$d_s$&Attributes of $D_s$&$d_t$&Attributes of $D_t$\\
	$X_s$&Records of $D_s$&$X_t$&Records of $D_t$\\
	$Y_s$&Labels of $D_s$&$Y_t$&Labels of $D_t$\\
	$F_s$&Forest on $D_s$&$F_t$&Forest on $D_t$\\
	$L_s$&Leaves of $F_s$&$L_t$&Leaves of $F_t$\\
	$w_i$&$i$th record&$r_i$&Label of $w_i$\\
	$C$&\texttt{\#} shared labels&\textbf{K}&Kernel matrix\\
	$\sigma$&Ridge regularization&$\mu$&MMD adaptive factor\\
	$\lambda$&MMD regularization&\textbf{M}&MMD matrix\\
	$\gamma$&Manifold regularization&\textbf{L}&Graph Laplacian matrix\\
	\bottomrule
	\end{tabular}
	\label{tab:notations}
\end{table}
\begin{definition}
\label{defdomain}
Domain~\cite{pan2010survey, long2014adaptation}: A domain $D$ is formed by a feature space $X$ having $d$ number of attributes and a marginal probability distribution $P(x)$, i.e., $D = {X, P(x)}$, where $x\in X$.
\end{definition}
\begin{definition}
\label{deftask}
Task~\cite{pan2010survey, long2014adaptation}: For a domain $D$, a task $T$ is formed by a set of labels $Y$ and a classification function $f(x)$, i.e., $T = \{Y, f(x)\}$, where $y\in Y$, and $f(x) = Q(y|x)$ can be considered as the conditional probability distribution.
\end{definition}

\begin{definition}
\label{def_distributiondivergence}
Distribution divergence: Let $D_s$ and $D_t$ be the source and target domains, respectively. Also, let $P_s(x_s)$ and $P_t(x_t)$ be the marginal distributions of $D_s$ and $D_t$, respectively, and $Q_s(y_s|x_s)$ and $Q_t(y_t|x_t)$ be the conditional distributions of $D_s$ and $D_t$, respectively. The distribution divergence can be defined with the setting where the marginal and conditional distributions of $D_s$ and $D_t$ are different i.e. $P_s(x_s) \neq P_t(x_t) \vee Q_s(y_s|x_s) \neq Q_t(y_t|x_t)$.
\end{definition}
\begin{definition}
\label{def_featurediscrepancy}
Feature discrepancy: Let $d_s$ and $d_t$ be the number of features (or attributes) of $D_s$ and $D_t$, respectively. The feature discrepancy means the dimensionality of $D_s$ and $D_t$ are different i.e. $d_s \neq d_t$. Moreover, the types of features of $D_s$ and $D_t$ are different.
\end{definition}

If the domains have different feature spaces or marginal distributions, i.e., $X_s \neq X_t \vee P_s(x_s) \neq P_t(x_t)$ then the domains can be considered as different i.e., $D_s \neq D_t$. Also, let $T_s$ and $T_t$ be the source task and target task, respectively. If the tasks have different labels or conditional distributions, i.e., $Y_s \neq Y_t \vee Q_s(y_s|x_s) \neq Q_t(y_t|x_t)$ then the tasks can be considered as different i.e., $T_s \neq T_t$.

\begin{definition}
\label{deftransferlearning}
Supervised Heterogeneous Transfer Learning (SHeTL): Given a source domain having labels i.e. $D_s = \left\{(x_{s_{1}}, y_{s_{1}}), . . . , (x_{s_{n_s}}, y_{s_{n_s}})\right\}$ and a target domain having labels i.e. $D_t = \left\{(x_{t_{1}}, y_{t_{1}}), . . . , (x_{t_{n_t}},y_{t_{n_t}})\right\}$, the goal of SHeTL is to build classifier $f$ on $D_t$ by learning a mapping $P:D_s \mapsto  D_t$, under the assumptions $n_s\gg n_t$, $X_s \cap  X_t=\emptyset$, $d_s \cap d_t=d$, $Y_s \cap Y_t=C$, $P_s(x_s) \neq P_t(x_t)$, and $Q_s(y_s|x_s) \neq Q_t(y_t|x_t)$.
\end{definition}

\section{Related Work}
\label{relatedwork}
A number of methods have been proposed for transfer learning (TL)~\cite{fang2022semi, yu2020label, zhen2019deep, kumar2020understanding, zhang2022transfer, zhang2021learning}. The methods that are most related to our proposed method TLF, are reviewed in this section. Based on feature discrepancy, TL methods can broadly be classified into 2 categories: Homogeneous TL (HoTL)  and Heterogeneous TL (HeTL).

\subsection{Homogeneous Transfer Learning (HoTL)}
\label{rw_hotl}
HoTL methods deal with the setting where (1) both source and target domains can be described by the same number and types of features, (2) the distributions of the domains are different~\cite{chen2016transfer}. Most of the methods generally build an initial model based on $D_s$ and then apply the model on $D_t$ to obtain pseudo-labels for the records of $D_t$. After that the methods iteratively minimize distribution divergence between $D_s$ and $D_t$, and update the pseudo-labels until convergence criteria are met~\cite{bruzzone2009domain, pan2011domain, kumar2020understanding, wang2020transfer}. 

An existing method called DASVM~\cite{bruzzone2009domain} initially builds a classifier by applying the support vector machine (SVM)~\cite{vapnik1998statistical} on $D_s$. At the first iteration, the records of $D_s$ are considered as the training set. The classifier is then used to estimate the labels of the records of $D_t$. The training set is then iteratively updated by adding $m$ records, that are classified with high confidence, from $D_t$ and removing $m$ records, that have less influence on the classifier, of $D_s$. In each iteration, the classifier is also retrained based on the current training set. Since DASVM ultimately builds the final classifier by considering only the records of $D_t$, a good classifier may not be built if $D_t$ contains insufficient records. Moreover, only the structural risk functional (SRF) is minimized in DASVM.  
 
In addition to SRF, Transfer Component Analysis (TCA)~\cite{pan2011domain} iteratively reduces the difference between joint distributions of $D_s$ and $D_t$ to ensure the domains are closer than the previous iteration. To measure the difference, TCA makes use of the maximum mean discrepancy (MMD)~\cite{pan2008transfer, liu2020learning}. TCA is improved in JDA~\cite{long2013transfer} which minimizes the difference of both conditional and marginal distributions.

Besides, ARTL~\cite{long2014adaptation} and MEDA~\cite{wang2020transfer} minimize the difference between joint distributions of $D_s$ and $D_t$ by using MMD, and maximize manifold consistency~\cite{belkin2006manifold} by considering the domains marginal distributions. However, a user requires to provide a value for $k$, to exploit the intrinsic property of the manifold consistency, which can be a difficult task of the user in real-world applications. 

Instead of adjusting the initial model, some other methods of this category first learn a domain-invariant space to represent the source and target domains, and then build a classifier on the domain-invariant space in order to classify test data~\cite{long2018transferable}. For example, DAN~\cite{long2018transferable} and JAN~\cite{long2017deep} apply deep neural networks to learn the domain-invariant space. Instead of a single common space, GFK~\cite{gong2014learning} finds a set of sub-spaces and then identifies the best sub-space by using the geodesic flow of the manifold consistency. 

However, due to feature discrepancy between $D_s$ and $D_t$, the above methods cannot be easily extended for solving such heterogeneous transfer learning tasks~\cite{chen2016transfer}. 
 
\subsection{Heterogeneous Transfer Learning (HeTL)}
\label{rw_hetl}
HeTL methods deal with the setting where (1) the dimensionality and features of source and target domains are different and disjoint, respectively, (2) the distributions of the domains are different~\cite{liu2020heterogeneous}. To alleviate the feature discrepancy and distribution divergence, the methods either learn a domain-invariant space~\cite{zhen2019deep, ren2019heterogeneous} for representing the data of $D_s$ and $D_t$ or learn a mapping to project such data~\cite{fang2022semi, yao2019heterogeneous} or explore correspondences to bridge the heterogeneity gap between $D_s$ and $D_t$~\cite{sukhija2019supervised, chen2016transfer}. 

\textit{To learn a common representation space}, DSCMR~\cite{zhen2019deep} and CLARINET~\cite{zhang2021learning} use a deep neural network based architecture for representing data of $D_s$ and $D_t$. The common representation space is used to bridge the heterogeneity gap between $D_s$ and $D_t$ and to compare the records from $D_s$ and $D_t$ directly by calculating similarities among the records. The similarity of the records having the same class label is higher than the similarity of the records having different class labels. Therefore, the class label of a test record can be predicted based on the similar records regardless of their feature types. Similarly, GLG~\cite{liu2020heterogeneous} utilizes the geodesic flow kernel to learn the common representation space.    

\textit{To learn a mapping}, CSFT~\cite{ren2019heterogeneous} finds a covariance-structured representation space by building two joint regression models: (i) $y^s_i=P^T\Phi(x^s_i)$ and (ii) $y^t_i=P^T\Phi(x^t_i)$. The regression models are then used to map the data of $D_s$ and $D_t$ into the shared space. After that CSFT builds a classifier by applying the SVM (or kNN) on the shared space and predicts the class value of a test record.     

A recent method called JMEA~\cite{fang2022semi} that makes use of a deep neural network to map the data of $D_s$ and $D_t$. The method assumes that $D_s$ contains huge labeled data and $D_t$ contains a few labeled and huge unlabeled data. It utilizes two networks for mapping: (1) from source to target, and (2) from labeled target to unlabeled target. Using the initial networks, the method obtains pseudo labels for the unlabeled data of $D_t$. The method takes the confidence of the pseudo labels into consideration to determine $k$ pseudo labels as the final labels of the unlabeled data. It then calculates the overall errors of the networks. The method iteratively updates the networks parameters and pseudo labels simultaneously until convergence criteria are met. Like JMEA, KHDA~\cite{fang2022semi} performs the mapping in two ways: from source to target and from labeled target to unlabeled target. However, KHDA considers a kernel matrix for the knowledge transfer instead of the neural network.

Similar to JMEA and KHDA, TOHAN~\cite{chi2021tohan} considers two deep networks simultaneously, where the first network transforms the data of $D_s$ into an intermediate domain and the second network focuses on learning from intermediate to target domain. Due to the transformation of data from $D_s$ into an intermediate domain, the privacy issue of the source data is solved.     

Another method called STN~\cite{yao2019heterogeneous} that builds an initial model based on the labeled data and then generates pseudo labels for the unlabeled data of $D_t$. It then iteratively minimizes the distribution divergence by using the MMD and regenerates the pseudo labels. The method repeats this process until the distribution divergence is minimized. 

Besides, SHOT~\cite{liang2020we} builds an initial model by applying a deep neural network on $D_s$. The model is then adjusted for $D_t$ without accessing the source data. The improvement SHOT is SHOT++~\cite{liang2021source} which calculates the confidence of predicting pseudo labels for the unlabeled data of $D_t$. The highly confident pseudo labels are considered in minimizing overall error and in updating parameters of the network . 
 
\textit{To explore correspondences}, TNT~\cite{chen2016transfer} uses a deep neural network to transform $D_s$ and $D_t$ into a shared space and then uses the neurons of the network to form a decision forest. Since the sigmoid function is used in each decision node of the forest, the split value of the node does not represent actual value (as shown in Fig.~\ref{fig:sampletree}). Thus, TNT may not be suitable for knowledge discovery purposes.

However, it is possible to discover patterns from a recent existing method called SHDA~\cite{sukhija2019supervised} which makes use of a forest algorithm on $D_s$ and $D_t$ to explore shared label information. Since our proposed framework, TLF is very relevant to SHDA we discuss it in detail as follows. 

SHDA first builds a forest (say 100 trees) by applying the Random Forest (RF)~\cite{breiman2001random} on $D_s$ and obtains a set of rules $R_s$ that produces $L_s$ leaves (see Fig.~\ref{fig:sampletree}). For every rule from the root to a leaf node, it calculates the contribution of each attribute of $D_s$. The contribution of an attribute $d_{s_{j}}\in d_s$ is calculated as $W_{s_{j}}= \sum^v_{i=1}{(\frac{1}{2})}^{v(i)}$  where $v$ is a list contains the level numbers at which the attribute $d_{s_{j}}$ is considered as the candidate split. Thus for all rules, it produces a contribution matrix $W_s^{[L_s\times d_s]}$. After that, for each leaf SHDA calculates the distributions of labels. If the size of the class attribute of $D_s$ is $C_s$ then SHDA produces a label distribution matrix, $Q_s^{[L_s\times C_s]}$. 

SHDA then refines $Q_s$ by removing the duplicate rows of the label distribution. For every duplicate row, the corresponding rows of the attributes contribution matrix $W_s$ are averaged. Similar to $D_s$, SHDA produces the attributes contribution matrix $W_t^{[L_t\times d_t]}$ and label distribution matrix, $Q_t^{[L_t\times C_t]}$ by building a forest from $D_t$, where $R_t$ is number of rules and $C_t$ is the size of the class attribute of $D_t$. The label distribution $L_t$ is also refined by removing the duplicate rows, and then the corresponding rows of the attributes contribution matrix $W_t$ are also averaged.

Using $Q_s$ and $Q_t$, SHDA identifies the identical or similar label distributions across $D_s$ and $D_t$. Let $N_p$ be the number of identical or similar label distributions across the source and target domains. These $N_p$ rows of $Q_s$ and $Q_t$ are called the pivots that act as the bridge between $D_s$ and $D_t$. The concept of bridging $D_s$ and $D_t$ is illustrated in Fig.~\ref{fig:basicconcept}.  

Since the attribute contribution matrices $W_s$ and $W_t$ are also refined, both of them have $N_p$ rows. The source projection matrix $P_s$ is obtained as~\cite{sukhija2019supervised}:
\begin{equation}
\begin{aligned}
\begin{array}{l l}
\displaystyle
\min_{P_s}{\frac{1}{N_p}} {\sum_{i=1}^{N_p} {\left\|W_t - W_s\cdot P_s\right\|_2^2}+\sum_{i=1}^{d_t} {\lambda \left\|P_{s_i}\right\|_1 }}& \\ 
	\text{such that}\quad P_{S_i}\geq 0 & 
\end{array}
\end{aligned}
\label{shda_ps}
\end{equation}
Using $P_s$, SHDA calculates the projected source data ($D_s\times P_s$) which are then appended to $D_t$ to build the final classifier. It is reported that SHDA outperforms existing state-of-the-art methods~\cite{sukhija2019supervised}. However, since SHDA projects the whole $D_s$ into $D_t$, the final classifier may suffer from negative transfer resulting in a low classification accuracy if the number of pivots between $D_s$ and $D_t$ is very low.

\subsection{Comparison with Existing Study}
\label{comparison_existing_study}

We now compare the proposed TLF with some existing studies. To the best of our knowledge, a number of existing methods are related to TLF, however, our work is significantly distinguished from them in the following ways:
\begin{itemize}
\item \textit{Comparison with homogeneous TL studies:} (1) DASVM~\cite{bruzzone2009domain} adjusts the classifier based on the pseudo labels and ultimately builds the final classifier based on the unlabeled target data. (2) ARTL~\cite{long2014adaptation}, MEDA~\cite{wang2020transfer}, JDA~\cite{long2013transfer} and BDA~\cite{wang2017balanced} utilize the pseudo labels to minimize the distribution divergence. (3) While TCA~\cite{pan2011domain} uses the MMD to align conditional distributions, GFK~\cite{gong2014learning} adopts the manifold consistency to align marginal distributions and does not take MMD into account. (4) To determine the nearest neighbors in manifold regularization, ARTL, MEDA and GFK require a user-defined $k$ which may be challenging for a user in real applications. 
\item \textit{Comparison with heterogeneous TL studies:}(1) DSCMR~\cite{zhen2019deep} and CLARINET~\cite{zhang2021learning} learn a common representation space. (2) JMEA~\cite{fang2022semi}, SHOT++~\cite{liang2021source} utilize the confidence of pseudo labels to minimize the loss function and iteratively perform classifier adaptation. (3) While TLF considers only the labeled target data in the training process, KHDA~\cite{fang2022semi} takes both labeled and unlabeled target data into account. KHDA simultaneously optimizes conditional and marginal distributions by using the MMD and Manifold. While KHDA gives equal importance to conditional and marginal distributions, TLF learns an adaptive factor dynamically to determine the importance of conditional and marginal distributions. Furthermore, KHDA requires a user-defined $k$ value for identifying nearest neighbors in manifold regularization. Furthermore, we argue that due to the use of pseudo-labels in the training process, KHDA may not be able to build a good model and may not perform well. Finally, while the classifier built by TLF can be used for knowledge discovery purposes, KHDA may not be suitable for this task since it does not generate any rules. (4) TNT~\cite{chen2016transfer} uses neurons of a neural network to construct a forest where the decision of a node is determined by the sigmoid function. However, TNT may not be suitable in discovering interesting patterns for knowledge discovery purposes. (5) SHDA~\cite{sukhija2019supervised} provides different importance to the different level nodes, whereas all decision nodes in a logic rule contribute equally for making the decision. During the transformation, it does not take the MMD and manifold into account to alleviate distribution divergence. Moreover, for a low number of correspondences, SHDA may not perform well since it may suffer from negative transfer due to mapping the entire source data into target domain.  
\end{itemize}
 
The existing methods therefore have room for further improvement. TLF, is capable of transferring knowledge from $D_s$ and $D_t$ by using dynamic distribution adaptation (DDA) and manifold regularization (MR) even if $d_s\neq d_t$, $Y_s\neq Y_t$ and $n_s\gg n_t$. A comparison on the properties of TLF and some existing methods is presented in Table~\ref{tab:model_comparison}.
\begin{table}[ht!]
	\scriptsize
	\centering
	\setlength{\belowcaptionskip}{-10pt}
	\setlength{\abovecaptionskip}{-5pt}	
	\caption{Comparison of TLF and some existing methods, where DDA, MR, and PL denote dynamic distribution adaptation, manifold regularization, and pseudo labels respectively.}
	\renewcommand\tabcolsep{3pt} 
	\begin{tabular}{lcccccc}
	\toprule
		Method [Ref.]&DDA?&MR?&$d_s\neq d_t$?&$Y_s\neq Y_t$?&$n_s\gg n_t$?&PL?\\
		\midrule
		DASVM~\cite{bruzzone2009domain}&No&No&No&No&No&Yes\\
	ARTL~\cite{long2014adaptation}&Yes&Yes&No&No&No&Yes\\
	MEDA~\cite{wang2020transfer}&Yes&Yes&No&No&No&Yes\\
	JDA~\cite{long2013transfer}&Yes&Yes&No&No&No&Yes\\
	BDA~\cite{wang2017balanced}&Yes&Yes&No&No&No&Yes\\
	TCA~\cite{pan2011domain}&Yes&No&No&No&No&Yes\\
	TOHAN~\cite{chi2021tohan}&No&No&No&No&No&Yes\\
	GLG~\cite{liu2020heterogeneous}&No&Yes&No&No&No&No\\
	KHDA~\cite{fang2022semi}&Yes&Yes&No&No&No&Yes\\
	SHOT++~\cite{liang2021source}&Yes&No&No&No&No&Yes\\
	TNT~\cite{chen2016transfer}&No&No&No&No&No&Yes\\
	SHDA~\cite{sukhija2019supervised}&No&No&Yes&Yes&Yes&No\\
	\textbf{TLF}~[ours]&Yes&Yes&Yes&Yes&Yes&No\\
	\bottomrule
	\end{tabular}
	\label{tab:model_comparison}
\end{table}
\section{Proposed Framework TLF}
\label{ourtechnique}

The motivation of TLF is presented in the introduction section. We now present our proposed a novel framework called TLF for supervised heterogeneous transfer learning (SHeTL). We then analyze the target domain error bound and computational complexity of TLF.

\subsection{General Framework}
\label{generalframework}

TLF consists of five main steps as follows (also shown in Algorithm~\ref{algo_tlf}).
\begin{figure*}[ht!]
\centering
  \setlength{\belowcaptionskip}{0pt}
	\setlength{\abovecaptionskip}{0pt}	
	\includegraphics[width=0.98\linewidth]{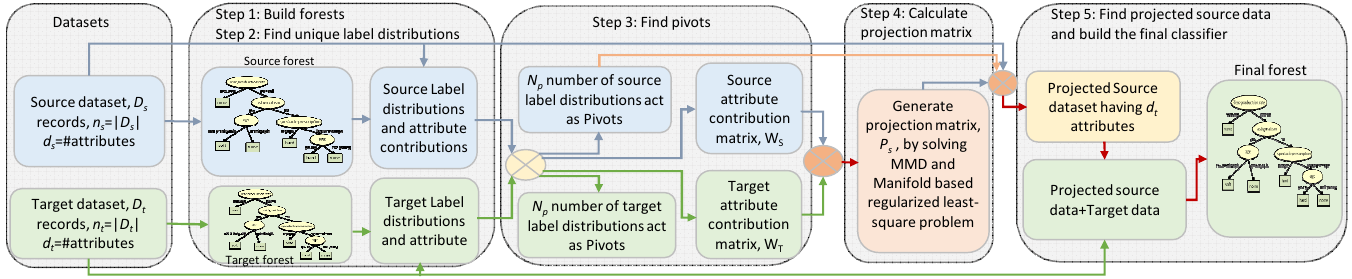}
	 \caption{Overall flow diagram of our proposed transfer learning framework, TLF.}
	 \label{fig:blockdiagram}
\end{figure*}

{\small  
\begingroup
\leftskip2em\rightskip2em
\setlength{\parindent}{0cm}
\vspace{.5em}
\hangindent=1cm
\textbf{Step-1:} Build forests and find label distributions and centroids from both source $D_s$ and target $D_t$ domains.

\hangindent=1cm
\textbf{Step-2:} Remove duplicate label distributions and find the aggregated values of the corresponding centroids. 

\hangindent=1cm
\textbf{Step-3:} Find pivots between source and target domains.

\hangindent=1cm
\textbf{Step-4:} Calculate projection matrix between source and target domains.

\hangindent=1cm
\textbf{Step-5:} Find and transfer projected source data to the target domain and build the final classifier.

\endgroup
}
We also present an overall flow diagram of TLF as shown in Fig.~\ref{fig:blockdiagram}. We now discuss the steps as follows.

\textbf{Step-1: Build forests and find label distributions and centroids from both source $D_s$ and target $D_t$ domains.}

In TLF, we take two datasets, namely source domain dataset $D_s$ and target dataset $D_t$ as input. $D_s$ has $n_s$ rows, $d_s$ attributes and $C_s$ labels, and $D_t$ has $n_t$ rows, $d_t$ attributes and $C_t$ labels. In this step, we build two forests $F_s$ and $F_t$ (where each forest has $\tau$ trees) by applying a decision forest algorithm such as Random Forest (RF)~\cite{breiman2001random} on $D_s$ and $D_t$, respectively as shown in the Step 1 of Algorithm~\ref{algo_tlf} and Fig.~\ref{fig:blockdiagram}. Using $F_s$, we find the leaves $L_s$ where a leaf $l^i_s\in L_s$ is represented by the $i$th rule of $F_s$. For the $i$th leaf $l^i_s$, we find records (from $D_s$) that belong to $l^i_s$. We then calculate the label distribution $v^i_s$, the centroid $w^i_s$ and centroid's label $r^i_s\in C_s$. The $v^{ij}_s$ is the probability of the $j$th label in $l^i_s$. If the $j$th attribute is a numerical attribute then the value of $w^{ij}_s$ is calculated as follows.
\begin{equation}
\begin{aligned}
	w^{ij}_s=u^{ij}_s+log(std)
\end{aligned}
\label{eq_centroid}
\end{equation}
where $u^{ij}_s$ and $std$ are the $j$th attribute mean and standard deviation, respectively, of the records that belong to leaf $l^i_s$. For categorical attributes, we use the mode values.  Thus, for all leaves of $F_s$ we obtain a label distribution matrix $V^{[L_s\times C_s]}_s$, a centroid matrix $W^{[L_s\times d_s]}_s$ and a matrix for centroids label $R^{[L_s\times 1]}_s$. Similarly, using $D_t$ and $F_t$, we find a label distribution matrix $V^{[L_t\times C_t]}_t$, a centroid matrix $W^{[L_t\times d_t]}_t$ and a matrix for centroids label $R^{[L_t\times 1]}_t$. The centroids are considered as the attribute contributions which are calculated differently than SHDA~\cite{sukhija2019supervised}. 

\textbf{Step-2: Remove duplicate label distributions and find the aggregated values of the corresponding centroids.}

Since label distribution of a leaf is calculated based on the number of records, of each label, belonging to the leaf, the existence of  duplicate label distributions at the leaves is common. In this step, we remove duplicates from $V_s$ and $V_t$ as shown in Step 2 of Algorithm~\ref{algo_tlf}. For every duplicate entry in $V_s$ and $V_t$, the corresponding centroids in $W_s$ and $W_t$ are aggregated. For numerical and categorical attributes we take the average and mode values, respectively. This procedure helps to keep attribute contributions for the aggregated vectors. Let $L^{\prime}_s$ and $L^{\prime}_t$ be the number of entries remained after the aggregation. Now, the matrices are:  $V^{[{L^{\prime}_s}\times {C_s}]}_s$, $V^{[{L^{\prime}_t}\times {C_t}]}_t$, $W^{[{L^{\prime}_s}\times {d_s}]}_s$, $W^{[{L^{\prime}_t}\times {d_t}]}_t$, $R^{[{L^{\prime}_s}\times 1]}_s$ and $R^{[{L^{\prime}_t}\times 1]}_t$.  

\textbf{Step-3: Find pivots between source and target domains.}

In most real-world HeTL scenarios, the feature discrepancy is a challenging issue since the correspondences between $D_s$ and $D_t$ are not available. To minimize the discrepancy, some common information between $D_s$ and $D_t$ can be identified to bridge the domains. Since we assume that both $D_s$ and $D_t$ share some label space, we exploit the shared label space to identify correspondences between domains. Thus, in Step 3, we determine the leaves with similar label distributions across the domains as shown in Step 3 of Algorithm~\ref{algo_tlf} and Fig.~\ref{fig:blockdiagram}. 

Let $S^{L^{\prime}_s\times L^{\prime}_t}$ be the similarity matrix of $V^{[{L^{\prime}_s}\times {d_s}]}_s$ and $V^{[{L^{\prime}_t}\times {d_t}]}_t$, where each value varies between 0 and 1. A lower value indicates a higher similarity. The similarity of two label distributions is computed using the Jensen Shannon divergence (JSD) method~\cite{saha2016multiple}. A source label distribution $V^i_s$ is said to be similar to a target label distribution $V^k_t$ if the divergence between them is less than 10\%. 

Since $D_t$ generally contains small labeled data, it is common to have a small number of common labels. So, we my have only a small number of similar labeled distributions across the domains. Each shared labeled distribution (i.e. $(V^i_s, V^k_t)$) is considered as a pivot that acts as the bridge between $D_s$ and $D_t$. Let $N_p$ be the number of pivots between $D_s$ and $D_t$.

For each pivotal label distribution, we also have the attribute contribution which is calculated in Step 1 and refined in Step 2. Based on the similar source and target label distributions, we then determine the attribute contribution matrices $W^{[N_p\times d_s]}_s$ and $W^{[N_p\times d_t]}_t$, and the associated label matrices $R^{[N_p\times 1]}_s$ and $R^{[N_p\times 1]}_t$. 

\textbf{Step-4: Calculate projection matrix between source and target domains.}

In this step, we calculate a projection matrix $P_s$ for $D_s$ so that the data can be projected to $D_t$ as shown in Step 4 of Algorithm~\ref{algo_tlf} and Fig.~\ref{fig:blockdiagram}. We calculate $P_s$ by applying structural risk minimization (SRM) principle and regularization theory~\cite{vapnik1998statistical} on the attribute contribution matrices $W^{[N_p\times d_s]}_s$ and $W^{[N_p\times d_t]}_t$. Let $J_s$ and $J_t$ be the source and target joint probability distributions, respectively and $P_s$ and $P_t$ be the source and target marginal distributions, respectively. To calculate $P_s$ we particularly, aim to optimize the following three objective functions:
\begin{enumerate}
	\item To minimize structural risk function on $W_s$ and $W_t$;
	\item To minimize distribution discrepancy between $J_s$ and $J_t$; and
	\item To maximize manifold consistency underlying $P_s$ and $P_t$.
\end{enumerate}
Let $D_c$ be the combined dataset of $W_s$, $R_s$, $W_t$ and $R_t$ i.e. $D_c = \left\{(w_{s_{1}}, r_{s_{1}}), . . . , (w_{s_{N_p}},r_{s_{N_p}}), (w_{t_{1}},r_{t_{1}}), . . . , (w_{t_{N_p}},r_{t_{N_p}})\right\}$. By ignoring source and destination, we can simply write $D_c$ as $\left\{(w_1, r_1), . . . , (w_z,r_z)\right\}$, where $z=2N_p$. Also let $f =\beta^T\phi(w)$ be the classifier having parameters $w$ and feature mapping function $\phi:w \mapsto \mathcal{H}$ that projects the actual feature vector to a Hilbert space $\mathcal{H}$. Using the SRM principle and regularization theory, the framework of TLF can be formulated as:
{\small
\begin{equation}
\begin{aligned}
	f=\argmin_{f \in \mathcal{H_K}} \sum^{z}_{i=1}{l(f(w_i),r_i)}+\sigma\left\|f\right\|^{2}_{K}\\+\lambda D_{f,K}(J_s,J_t)+\gamma M_{f,K}(P_s,P_t)
\end{aligned}
\label{eq_generalframework}
\end{equation}
}
where $K$ is the kernel function represented by $\phi$ such that $\left\langle\phi (w_i), \phi (w_i)\right\rangle = K (w_i, w_j)$, and $l(.,.)$, $\left\|f\right\|^{2}_{K}$ and $D_{f,K}(.,.)$ represent loss function, ridge regularization and dynamic distribution adaptation (DDA), respectively. Moreover, we introduce $M_{f,K}(P_s,P_t)$ as a Laplacian regularization to further exploit intrinsic properties of nearest points in Manifold $\mathcal{G}$~\cite{belkin2006manifold}. The symbols $\sigma$, $\lambda$, and $\gamma$ are positive regularization parameters for ridge, DDA, and Manifold, respectively. 

In Eq.~\eqref{eq_generalframework}, the loss function $l(.)$ measures the fitness of $f$ for classifying labels on $D_c$. We use SVM~\cite{vapnik1998statistical} to solve $l(.)$. Moreover, since the kernel mapping $\phi:w \mapsto \mathcal{H}$ may have infinite dimensions, we solve Eq.~\eqref{eq_generalframework} by reformulating it with the following Representer Theorem. 
\begin{theorem}
(Representer Theorem)~\cite{long2014adaptation, vapnik1998statistical} If $K$ is a kernel function induced by $\phi$ then the optimization problem in Eq.~\eqref{eq_generalframework} for $D_c$, can be represented by
{\small
\begin{equation}
\begin{aligned}
 f(w)=\sum^{z}_{i=1}{\alpha_i{K(w_i,w)}} \text{~and~} \beta=\sum^{z}_{i=1}{\alpha_i{\phi(w_i)}}
\end{aligned}
\label{eq1_representertheorem}
\end{equation}
}
where $\alpha_i$ is a coefficient.
\label{theorem_representer}
\end{theorem}
Using the SRM principle and equations~\eqref{eq1_representertheorem}, we can reformulate loss function, ridge regularization of Eq.~\eqref{eq_generalframework} as follows.
{\small
\begin{equation}
\begin{aligned}
	f=\argmin_{f\in\mathcal{H}} \left\|R-\alpha^TK\right\|^{2}_{F}+\sigma(tr(\alpha^T{K}\alpha))
\end{aligned}
\label{eq_srm_minimize}
\end{equation}
}
We now discuss the process of minimizing distribution divergence through dynamic distribution adaptation and manifold regularization as follows.

\textbf{Dynamic Distribution Adaptation (DDA):} Since the joint distributions $J_s$ and $J_t$ are different, the loss function $l(.)$ with only the ridge regularization in Eq.~\eqref{eq_generalframework} may not generalize well to $W_t$~\cite{long2014adaptation}. Thus, we aim to minimize joint distribution divergence 1) between the marginal distributions $P_s$ and $P_t$, and 2) between the conditional distributions $Q_s$ and $Q_t$, simultaneously. 

Following ARTL~\cite{long2014adaptation} and MEDA~\cite{wang2020transfer}, we adopt MMD~\cite{pan2008transfer, liu2020learning} to handle the distribution divergence. Based on the samples drawn from two domains, MMD is capable of determining whether two distributions are different or not. The MMD between marginal distributions $P_s$ and $P_t$ is computed as 
{\scriptsize
\begin{equation}
\begin{aligned}
	MMD^2_{\mathcal{H}}(W_s,W_t)= \left\|\frac{1}{N_p}\sum^{N_p}_{i=1}{\phi(w_i)}-\frac{1}{N_p}\sum^{z}_{j=N_p}{\phi(w_j)}\right\|^{2}_{\mathcal{H}}
\end{aligned}
\label{eq_marginal_mmd1}
\end{equation}
}
where $\phi:w \mapsto \mathcal{H}$ is a feature mapping. The MMD can be regularized for the classifier $f$ as
{\scriptsize
\begin{equation}
\begin{aligned}
	D_{f,K}(P_s,P_t)=\left\|\frac{1}{N_p}\sum^{N_p}_{i=1}{f(w_i)}-\frac{1}{N_p}\sum^{z}_{j=N_p}{f(w_j)}\right\|^{2}_{\mathcal{H}}
\end{aligned}
\label{eq_marginal_mmd2}
\end{equation}
}
where $f =\beta^T\phi(w)$.

Since we have labeled data in both domains, we now minimize the distance $D_{f,K}(P_s,P_t)$ between conditional distributions $Q_s$ and $Q_t$. Based on the intra-class centroids of two distributions $Q_s(w_s|r_s)$ and $Q_t(w_t|r_t)$, the MMD of a class $c\in C$ is computed as 
{\scriptsize
\begin{equation}
\begin{aligned}
	D^{(c)}_{f,K}(Q_s,Q_t)=\left\|\frac{1}{|W^{(c)}_s|}\sum_{w_i\in W^{(c)}_s}{f(w_i)}-\frac{1}{|W^{(c)}_t|}\sum_{w_j\in W^{(c)}_t}{f(w_j)}\right\|^{2}_{\mathcal{H}}
\end{aligned}
\label{eq_conditional_mmd}
\end{equation}
}
where $W^{(c)}_s=\left\{w_i: w_i\in W_s \wedge r(w_i)=c\right\}$ is the set records associated to class $c$ in $D_s$ and $W^{(c)}_t=\left\{w_j: w_j\in W_t \wedge r(w_j)=c\right\}$ is the set records associated to class $c$ in $D_t$. If $|W^{(c)}_k|=0$ then we ignore the part of $|W_k|$ in Eq.~\eqref{eq_conditional_mmd}.

In real-world scenarios, the importance of marginal ($P$) and conditional ($Q$) are different~\cite{wang2020transfer}. Thus, the divergence for joint joint distributions $J_s$ and $J_t$ is adapted dynamically by integrating Eq.~\eqref{eq_marginal_mmd2} and Eq.~\eqref{eq_conditional_mmd} as follows.

{\scriptsize
\begin{equation}
\begin{aligned}
	D_{f,K}(J_s,J_t)=(1-\mu)D_{f,K}(P_s,P_t)+\mu\sum^{C}_{c=1}{D^{(c)}_{f,K}(Q_s,Q_t)}
\end{aligned}
\label{eq_joint_mmd}
\end{equation}
}
where $\mu\in[0,1]$ is the adaptive factor which is calculated dynamically by following MEDA~\cite{wang2020transfer}. $\mu\rightarrow 0$ indicates the distribution distance between the domains is large, and thus, the marginal distribution adaptation is important. Besides, $\mu\rightarrow 1$ indicates the class distribution is dominant than feature distribution, and thus, the conditional distribution adaptation is important. 

Using the Representer Theorem (see Theorem~\ref{theorem_representer} and Eq.~\eqref{eq1_representertheorem}), Eq.~\eqref{eq_joint_mmd} becomes
{\scriptsize
\begin{equation}
\begin{aligned}
D_{f,K}(J_s,J_t)=tr(\alpha^T{K}{M}{K}\alpha)
\end{aligned}
\label{eq_joint_mmd2}
\end{equation}
} 
where $M^{[z\times z]}=(1-\mu)M_0+\mu\sum^{C}_{c=1}{M_c}$ is the MMD matrix. The components of the MMD matrix are calculated as follows.
{\small
\begin{equation}
\begin{aligned}
  (M_0)_{ij} = \left\{
  \begin{array}{l l}
    \frac{1}{N^2_p} & \text{$w_i, w_j\in{W_s} \vee w_i, w_j\in{W_t}$}\\		
		\frac{-1}{N^2_p} & \text{otherwise}\\
  \end{array} \right.
  ;\forall {i,j}
\end{aligned}
\label{eq_mmd_matrix_m0}
\end{equation}
} 
{\small
\begin{equation}
\begin{aligned}
  (M_c)_{ij} = \left\{
  \begin{array}{l l}
    \frac{1}{n^2_c} & \text{$w_i, w_j\in{W^{(c)}_s}$}\\	
		\frac{1}{m^2_c} & \text{$w_i, w_j\in{W^{(c)}_t}$}\\
		\frac{-1}{{n^2_c}{m^2_c}} & \text{$w_i, w_j\in{W^{(c)}_s} \vee w_i, w_j\in{W^{(c)}_t}$}\\	
		0 & \text{otherwise}\\
  \end{array} \right.
	;\forall {i,j}
\end{aligned}
\label{eq_mmd_matrix_mc}
\end{equation}
} 
where $n_c=|W^{(c)}_s|$ and $m_c=|W^{(c)}_t|$.

\textbf{Manifold Regularization:} In addition to DDA, we further exploit intrinsic properties of the marginal distributions $P_s$ and $P_t$ by manifold regularization~\cite{belkin2006manifold} due to its effectiveness in mitigating influence of feature distortion~\cite{long2014adaptation, wang2020transfer}. The features of a domain can be transformed into manifold space which is suitable for exploiting more detailed property and structure of the domain. Considering the manifold assumption~\cite{belkin2006manifold}, the manifold regularization is computed as
\begin{equation}
\begin{aligned}
\begin{array}{l l}
\displaystyle
	M_{f,K}(P_s,P_t)&=\sum^{z}_{i,j=1}{{(f(w_i)- f(w_j))}^2{B_{ij}}} \\
	&=\sum^{z}_{i,j=1}{{f(w_i)}{L_{ij}}{f(w_j)}}
\end{array}	
\end{aligned}
\label{eq_manifold1}
\end{equation}
where $L^{[z\times z]}$ is the normalized Laplacian regularization matrix and $B$ is the graph affinity matrix which is calculated as

{\small
\begin{equation}
\begin{aligned}
  B_{ij} = \left\{
  \begin{array}{l l}
    \cos(w_i, w_j) & \text{if $w_i\in{k(w_j)} \vee w_j\in{k(w_i)}$}\\	
		0 & \text{otherwise}\\
  \end{array} \right.
\end{aligned}
\label{eq_manifold_sim_matrix}
\end{equation}
} 
where $\cos(w_i, w_j)=\frac{w_i.w_j}{\|w_i\|.\|w_j\|}$ is a similarity function which is used to calculate distance between two records $w_i$ and $w_j$. It is worth noting that any similarity functions such as guassian, laplacian and cosine, can be incorporated in TLF. We empirically test the effectiveness of TLF, by using three kernels separately, on a real-world dataset Reuters-215782 (RT)~\cite{transferlearningRsc} which has 3 cross-domain pairs as follows: orgs$\mapsto$people, orgs$\mapsto$places and people$\mapsto$places (see Section~\ref{datapreparation} for details). For each cross-domain pair, we separately use guassian, laplacian and cosine kernels in Eq.~\eqref{eq_manifold_sim_matrix} and calculate the corresponding classification accuracy. We present the results in Fig.~\ref{fig:kernels_comparison} which shows that TLF performs the best when cosine kernel is used in Eq.~\eqref{eq_manifold_sim_matrix}. Therefore, we use cosine kernel in TLF to calculate distance between two records $w_i$ and $w_j$.
\begin{figure}[ht!]
\centering
  \setlength{\belowcaptionskip}{0pt}
	\setlength{\abovecaptionskip}{0pt}	
    \includegraphics[width=0.90\linewidth]{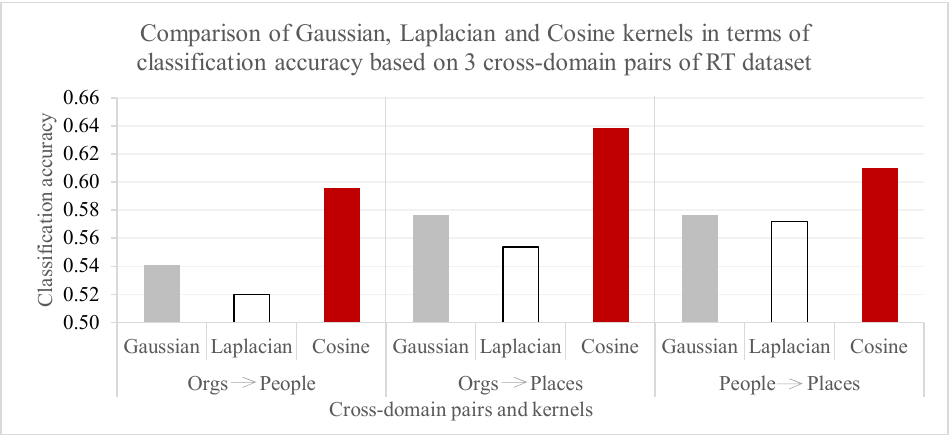}
   \caption{Comparison of Gaussian, Laplacian and Consine kernels in terms of classification accuracy for ``Orgs$\mapsto$People'', ``Orgs$\mapsto$places'', and ``people$\mapsto$places'' cross-domain pairs on RT dataset.}
	    \label{fig:kernels_comparison}
\end{figure}

In Eq.~\eqref{eq_manifold_sim_matrix}, $k(w_j)$ represents the $k$ nearest-neighbors (kNN) of the record $w_j$. While existing methods including ARTL~\cite{long2014adaptation} and MEDA~\cite{wang2020transfer} use a user-defined value for $k$, we propose a novel approach to determine $k$ automatically so that the intrinsic properties of a record can be exploited more accurately resulting in a high transfer performance.

In TLF, we determine $k$ NN of a record $w_i$ having label $r_i$ as follows. We first calculate distance between $w_i$ and other records one by one and then sort the records based on the distance in ascending order. We set the minimum of value of $k$ to 4. That is, the first 
four records are added to kNN set. Following the order, we continue to add a record $w_u$ to kNN set iteratively until $r_i\neq{r_u}; \forall u$. Fig.~\ref{fig:manifold_k} illustrates the difference between user-defined and automatic $k$, the automatic one finds the nearest-neighbors of a record more accurately.
\begin{figure}[ht!]
\centering
  \setlength{\belowcaptionskip}{0pt}
	\setlength{\abovecaptionskip}{0pt}	
	    \subfigure[User-defined $k$]
	    {
	        \includegraphics[width=0.40\linewidth]{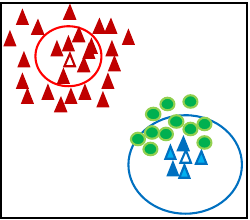}
	        \label{fig:userdefinedk}
	    }
	    \subfigure[Automatic $k$]
	    {
	        \includegraphics[width=0.40\linewidth]{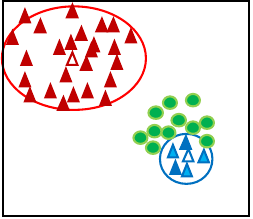}
	        \label{fig:automatick}
	    }
	    \caption{Justification of determining $k$ value automatically in manifold regularization.}
	    \label{fig:manifold_k}
\end{figure}

We now test the influences of automatic $k$ on the RT dataset. We first apply TLF on each pair for different user-defined $k$ values such as 4, 8, 16, 32 and 64, and then calculate classification accuracy for each $k$. After that we apply TLF on each pair where $k$ is determined automatically and then calculate classification accuracy. We present the results of the people$\mapsto$place pair in Fig.~\ref{fig:justifyautok}. The figure clearly shows the effectiveness of automatic $k$ since TLF achieves the best results when we determine $k$ automatically. We also obtain similar results for other two pairs.
\begin{figure}[ht!]
\centering
  \setlength{\belowcaptionskip}{0pt}
	\setlength{\abovecaptionskip}{0pt}	
    \includegraphics[width=0.70\linewidth]{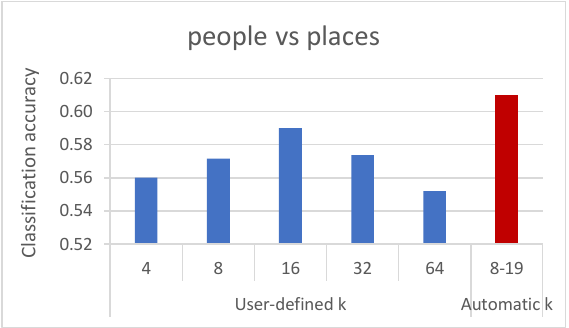}
   \caption{Justification of automatic determination of k value for the cross-domain pair ``people$\mapsto$place'' on RT dataset in terms of classification accuracy.}
	    \label{fig:justifyautok}
\end{figure}

The Laplacian matrix, $L$ of Eq.~\eqref{eq_manifold1} is calculated as
{\small
\begin{equation}
\begin{aligned}
L=I-D^{-\frac{1}{2}}BD^{-\frac{1}{2}} 
\end{aligned}
\label{eq_manifold_L}
\end{equation}
} 
where $D$ is a diagonal matrix with each element $D_{ii}=\sum^z_{j=1}{B_{ij}}$. Thus, by incorporating Eq.~\eqref{eq1_representertheorem} into Eq.~\eqref{eq_manifold1}, we get
{\small
\begin{equation}
\begin{aligned}
M_{f,K}(P_s,P_t)=tr(\alpha^T{K}{L}{K}\alpha)
\end{aligned}
\label{eq_manifold_all}
\end{equation}
} 
Using equations~\eqref{eq_srm_minimize}, \eqref{eq_joint_mmd2} and \eqref{eq_manifold_all}, we can reformulate $f$ in Eq.~\eqref{eq_generalframework} as follows.
{\small
\begin{equation}
\begin{aligned}
	f=\argmin_{f\in\mathcal{H}} \left\|R-\alpha^TK\right\|^{2}_{F}+\sigma(tr(\alpha^T{K}\alpha))\\+\lambda(tr(\alpha^T{K}{M}{K}\alpha))+\gamma(tr(\alpha^T{K}{L}{K}\alpha))
\end{aligned}
\label{eq_final_f}
\end{equation}
}
Setting derivative $\frac{\partial f}{\partial\alpha}=0$, we obtain the solution as
{\small
\begin{equation}
\begin{aligned}
	\alpha=\sigma{I}+(\lambda{M}+\gamma{L})K
\end{aligned}
\label{eq_final_alpha}
\end{equation}
}
We now transform the attribute contribution matrices $W^{[N_p\times d_s]}_s$ and $W^{[N_p\times d_t]}_t$ into $G^{[z\times d_s]}_s$ and $G^{[z\times d_t]}_t$, respectively as follows.
{\small
\begin{equation}
\begin{aligned}
  (G_s)_{ij} = \left\{
  \begin{array}{l l}
    e_{ij} & \text{if $e_{ij}=w_{ij} \wedge w_i\in{W_s}$}\\	
		0 & \text{otherwise}\\
  \end{array} \right.
	;\forall i,j
\end{aligned}
\label{eq_src_transform}
\end{equation}
} 
{\small
\begin{equation}
\begin{aligned}
  (G_t)_{ij} = \left\{
  \begin{array}{l l}
    e_{ij} & \text{if $e_{ij}=w_{ij} \wedge w_i\in{W_t}$}\\	
		0 & \text{otherwise}\\
  \end{array} \right.
	;\forall i,j
\end{aligned}
\label{eq_tgt_transform}
\end{equation}
} 
Using equations~\eqref{eq_final_alpha}, \eqref{eq_src_transform} and \eqref{eq_tgt_transform}, we obtain the projection matrix $P_s$ as follows.
{\small
\begin{equation}
\begin{aligned}
	P_s=G^T_s{\alpha}{G_t}
\end{aligned}
\label{eq_final_Ps}
\end{equation}
}
\textbf{Step-5: Find and transfer projected source data to the target domain and build the final classifier.}

In this step, we first identify transferable data $D^{\prime}_s$ from  $D_s$ by using the projection matrix $P_s$ (obtained in Eq.~\eqref{eq_final_Ps}) and source pivots $V_s$ (obtained in Step 3). We then project $D^{\prime}_s$ into $D^{\prime\prime}_s$ i.e. $D^{\prime{\prime}}_s=D^{\prime}_s\times{P_s}$. After that we merge $D^{\prime\prime}_s$ with $D_t$ and obtain $D^{\prime}_t$ (see Step 5 of Algorithm~\ref{algo_tlf} and Fig.~\ref{fig:blockdiagram}). We finally build the final forest $F$ by applying RF on $D^{\prime}_t$. The final classifier is then used to classify unlabeled target data.

\IncMargin{0.5em}
\begin{algorithm}
\SetKwBlock{StepZero}{Initialize:}{end}
\SetKwBlock{StepOne}{Step 1:}{end}
\SetKwBlock{StepTwo}{Step 2:}{end}
\SetKwBlock{StepThree}{Step 3:}{end}
\SetKwBlock{StepFour}{Step 4:}{end}
\SetKwBlock{StepFive}{Step 5:}{end}

\SetKwInOut{Input}{Input}\SetKwInOut{Output}{Output}
{\scriptsize 
\Input{Source dataset $D_s$, target dataset $D_t$.
}
\Output{Classifier $F$.}
\BlankLine
\DontPrintSemicolon

\StepOne{
Build two decision forests $F_s$ and $F_t$ by applying an existing decision forest algorithm such as RF~\cite{breiman2001random} on $D_s$ and $D_t$, respectively, where $|F_s|$=$|F_t|$=$\tau$ trees;\;
Find source leaves, $L_s \leftarrow findLeaves(F_s)$;\;
Find source label distributions, $V_s \leftarrow findClassDistribution(D_s, L_s)$;\;
Find source centroids, $W_s \leftarrow findCentroids(D_s, L_s)$;\;
Find source centroid's labels, $R_s \leftarrow findCentroidLabel(D_s, L_s)$;\;
Find target leaves, $L_t \leftarrow findLeaves(F_t)$;\;
Find target label distributions, $V_t \leftarrow findClassDistribution(D_t, L_t)$;\;
Find target centroids, $W_t \leftarrow findCentroids(D_t, L_t)$;\;
Find target centroid's labels, $R_t \leftarrow findCentroidLabel(D_t, L_t)$;\;
}

\StepTwo{ 
Remove duplicates from $V_s$ and $V_t$. For every duplicate entry in $V_s$ and $V_t$, the corresponding centroids in $W_s$ and $W_t$, and centroid labels $R_s$ and $R_t$ are aggregated.
}

\StepThree{
Calculate similarity matrix, $S\leftarrow calculateSimilarityByJSD(V_s, V_t)$ between source and target label distributions;\; 
Find source pivots, $V_s\leftarrow findPivots(V_s, S)$ and associated centroids $W_s$ and labels $R_s$;\;  
Find target pivots, $V_t\leftarrow findPivots(V_t, S)$ and associated centroids $W_t$ and labels $R_t$;\; 
}
\StepFour{	
Find projection matrix, $P_s\leftarrow findProjectionMatrix(W_s, R_S, W_t, R_t)$ by considering ridge, maximum mean discrepancy (MMD) and manifold regularizations as shown in Eq.~\eqref{eq_final_Ps}\; 
}
\StepFive{
	Find transferrable source data, $D^{\prime}_s\leftarrow findProjectedSourceData(V_s, D_s, P_s)$;\;
	Projected source data, $D^{\prime\prime}_s\leftarrow D^{\prime}_s\times P_s$;\;
	New target dataset, $D^{\prime}_t\leftarrow merge(D^{\prime\prime}_s, D_t)$;\;
	Build the final decision forest, $F$ by applying the RF on $D^{\prime}_t$; \;
	Return $F$.
}
}
\caption{LearnTLF()}\label{algo_tlf}
\end{algorithm}\DecMargin{0.5em}

\subsection{Computational Complexity}
\label{computationalcomplexity}

We now calculate the complexity of TLF (as shown in Algorithm~\ref{algo_tlf}). As input, TLF takes two datasets: $D_s$ has $n_s$ records and $d_s$ attributes, and $D_t$ has $n_t$ records and $d_t$ attributes. Let $D_s$ and $D_t$ have $C$ common labels and $n(=n_s+n_t)$ records.

In Step 1, we build two forests $F_s$ and $F_t$ having $\tau$ trees each. Since $n_s\gg n_t$, the complexity of building forests is $O({\tau}{d_s}{ n^2_s}{log(n_s)})$~\cite{sukhija2019supervised}. We then identify shared label distributions that has a complexity $O({\tau}^2{n_s}{n_t}C)$. 

The complexity of Step 2 is $C(l^2_s+l^2_t)$ where $l_s$ and $l_t$ are the leaves of $F_s$ and $F_t$, respectively. In Step 3, $N_p$ common label distributions across the domains are identified. The complexity of Step 3 is $O(N_p(d_s+d_t))$. 

The complexity of Step 4 consists of three parts: (i) solving the linear system of Eq.~\eqref{eq_srm_minimize} using LU decomposition that has a complexity $O((z)^3)$~\cite{long2014adaptation} for TLF, where $z=2N_p$, (ii) for constructing the graph Laplacian matrix L, TLF needs $O((d_s+d_t){z^2})$, which is performed once, and (iii) for constructing the kernel matrix $K$ and aggregate
MMD matrix $M$, TLF requires $O(C{z^2})$. In Step 5, TLF needs $O({n_s}{d_s}{d_t})$ for transferring the source dataset to target domain. Finally, in Step 6, TLF builds the final decision forest which requires $O(\tau{d_t}{n^2}{log(n)})$. Typically, $n_t$, $N_p$, $C$, $z$, $l_s$ and $l_t$ are very small, especially compared to $n$. Therefore, the overall complexity of TLF is $O({n_s}{d_s}{d_t}+\tau{d_t}{n^2}log(n))$.

\section{Experiments}
\label{experiments}

\subsection{Data Preparation}
\label{datapreparation}

We apply TLF and fourteen existing methods on seven real datasets which are shown at a glance in Table~\ref{tab:Datasets}. 

\begin{table*}[ht!]
	\small
	\centering
	\setlength{\belowcaptionskip}{-10pt}
	\setlength{\abovecaptionskip}{-5pt}	
	\caption{datasets at a glance.}
	\renewcommand\tabcolsep{3pt} 
	\begin{tabular}{lrcccl}
	\toprule
		dataset [Ref.]&\textit{\#}Records& \textit{\#}Attributes&\textit{\#}Classes& Sub-domain & Area\\
		\midrule
			PIE~\cite{transferlearning.xyz}&11554&1024&68&PIE5, PIE7, PIE9, PIE27, PIE29&Face recognition\\
			Office+Caltech (OC)~\cite{transferlearning.xyz}&2533&800&10&amazon, Caltech10, dslr, webcam&Object recognition\\
			20-Newsgroups (NG)~\cite{transferlearningRsc}&40252&~10000&2&C, S, R, T&Text Classification\\
			VLSC~\cite{transferlearning.xyz}&18070&4096&5&C, I, L, S, V&Image classification\\
			Reuters-21578 (RT)~\cite{transferlearningRsc}&6570&~4000&3&Orgs,People,Places&Text Classification\\
			MNIST+USPS (MU)~\cite{transferlearning.xyz}&3800&256&10&USPS, MNIST&Object recognition\\
		  COIL20 (COIL)~\cite{transferlearning.xyz}&1440&1024&20&COIL1, COIL2&Object recognition\\
		\bottomrule
	\end{tabular}
	\label{tab:Datasets}
\end{table*}

PIE~\cite{transferlearning.xyz} has 11554 handwritten digits from 5 different subcategories: PIE5, PIE7, PIE9, PIE27 and PIE29. We generate 20 combinations of source-target pairs as follows: PIE5$\mapsto$PIE7, PIE5$\mapsto$PIE9, PIE5$\mapsto$PIE27, PIE5$\mapsto$PIE29, PIE7$\mapsto$PIE5, and so on. Office+Caltech (OC)~\cite{transferlearning.xyz} has 4 different subcategories: amazon, Caltech10, dslr and webcam. We generate 12 combinations of source-target pairs as follows: amazon$\mapsto$Caltech10, amazon$\mapsto$dslr, amazon$\mapsto$webcam, Caltech10$\mapsto$amazon, and so on.

20-Newsgroups (NG)~\cite{transferlearningRsc} has approximately 20,000 documents distributed evenly in 4 different subcategories: C, S, R, and T. Following the approach in ARTL~\cite{long2014adaptation}, we have 5 combinations of source-target pairs as follows: CS-src$\mapsto$CS-tgt, CT-src$\mapsto$CT-tgt, RS-src$\mapsto$RS-tgt, RT-src$\mapsto$RT-tgt and ST-src$\mapsto$ST-tgt. Reuters-215782 (RT)~\cite{transferlearningRsc} has three top categories orgs, people, and place. Following the approach in ARTL~\cite{long2014adaptation}, we can construct 3 cross-domain pairs as follows: orgs$\mapsto$people, orgs$\mapsto$place and people$\mapsto$place.

VLSC~\cite{transferlearning.xyz} consists of 18070 images from 5 different subcategories: C, I, L, S, and V. We generate 20 combinations of source-target pairs as follows: C$\mapsto$I, C$\mapsto$L, C$\mapsto$S, C$\mapsto$V, I$\mapsto$C, and so on. Similarly, MNIST+USPS (MU) and COIL~\cite{transferlearning.xyz} have 4 different subcategories from which we generate 12 combinations of source-target pairs, separately.

\subsection{Experimental Setup}
\label{experimentalsetup}

\subsubsection{Baseline Methods}
\label{baselinemethods}
We compare the performance of TLF with fourteen state-of-the-art machine learning and transfer learning techniques, comprising twelve transfer learning and two baseline (non-transfer learning) algorithms, namely SHOT++~\cite{liang2021source}, TNT~\cite{chen2016transfer}, TOHAN~\cite{chi2021tohan}, SHDA~\cite{sukhija2019supervised}, MEDA~\cite{wang2020transfer, wang2018visual}, TCA~\cite{pan2011domain}, BDA~\cite{wang2017balanced}, 	CORAL~\cite{sun2015return},	EasyTL~\cite{wang2019easytl}, JDA~\cite{long2013transfer}, ARTL~\cite{long2014adaptation}, GFK~\cite{gong2014learning}, RF~\cite{breiman2001random}, and SysFor~\cite{islam2011knowledge}.

We implement TLF in the Java programming language using the Weka APIs~\cite{witten2016data}. Upon acceptance, the source code of TLF will be made available at GitHub. We also implement an existing technique called SHDA~\cite{sukhija2019supervised}. For implementing the RF and SysFor, we use the Java code from the Weka platform~\cite{witten2016data}. For SHOT++\footnote{SHOT++ code: \url{https://github.com/tim-learn/SHOT-plus}}, TNT\footnote{\url{https://github.com/wyharveychen/TransferNeuralTrees}} and TOHAN\footnote{TOHAN code: \url{https://github.com/Haoang97/TOHAN}}, we use the code available at GitHub. All other methods are already available in the MEDA framework~\cite{wang2018visual}. We use the default settings of the methods while running the experiment. 

\subsubsection{Implementation Details}
\label{implementationdetails}

In our experiment, we use a number of parameters for TLF and existing techniques. Forest algorithms require two common parameters, namely number of trees ($\tau$) and minimum leaf size. $\tau$ is set to 10 and the minimum leaf size for large datasets (i.e. the size of the dataset is greater than 10000) is set to 50, otherwise it is set to 20. For TLF, we use $\lambda=0.01$, $\theta=5$ and $\gamma=0.001$ for manifold, MMD and ridge regularizations, respectively. In our experimentation we use the majority voting~\cite{jamali2011majority} method to calculate final output of a forest.

\subsection{Impact of Target Dataset Size}
\label{impact_targetdata}

In this section, we evaluate how the size of the target dataset affects the classification performance of a transfer learning method. We also explore the necessity of transfer learning across domains. From the original target dataset $D^{ori}_t$, of each source$\mapsto$target pair, we create a new target dataset $D^{new}_t$ and a test dataset $D^{test}_t$, where $D^{new}_t$ contains x\% ($i.e. x=1, 2, 3, 4, 5, 10, 15,...,95$) randomly selected records of $D^{ori}_t$ and $D^{test}_t$ contains the remaining records of $D^{ori}_t$. For each x value, we first build a classifier by applying a transfer learning method on $D_s$ and $D^{new}_t$, and then calculate the classification accuracy by applying the classifier on $D^{test}_t$.

To explore the affects of target dataset size, we carry out an experimentation on the VLSC and COIL datasets and present the average classification accuracies of 10 runs of RF-S, RF-T and TLF in Fig.~\ref{fig:justifytgtratio}. For RF-S, we build a forest by applying RF on the source data, $D_s$ and use the forest to classify the test data, $D^{test}_t$. Similarly, for RF-T, we build a forest by applying RF on the target data, $D^{new}_t$ and use the forest to classify the test data, $D^{test}_t$. The x-axis of the figure shows the size of $D^{new}_t$ and corresponding ratio (\% of records) of $D^{ori}_t$. 

\begin{figure}[ht!]
\centering
  \setlength{\belowcaptionskip}{0pt}
	\setlength{\abovecaptionskip}{0pt}		    
	    \subfigure[VLSC dataset]
	    {
	        \includegraphics[width=0.80\linewidth]{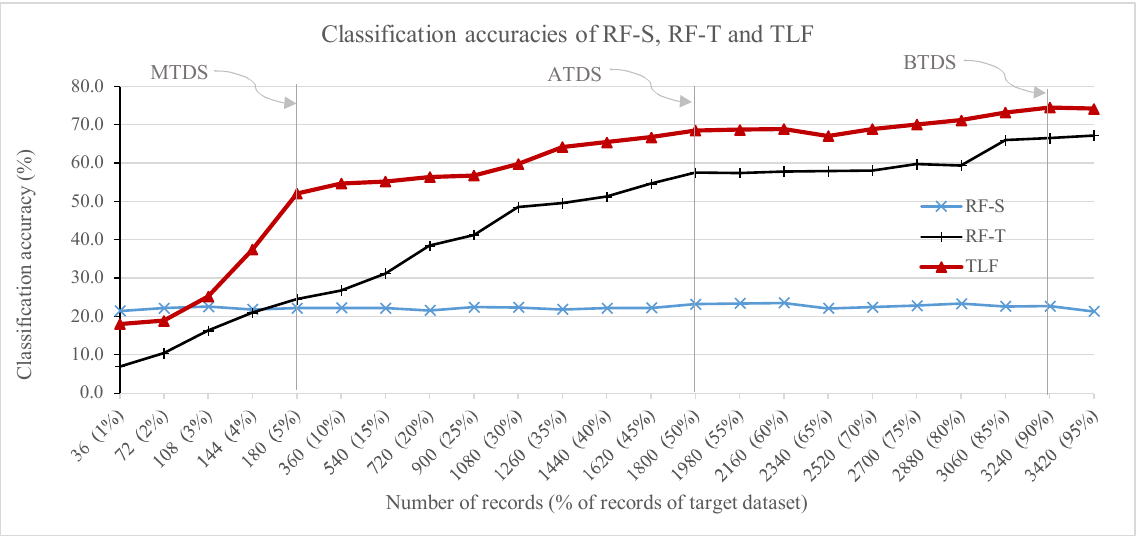}
	        \label{fig:justifytgtratio_vlsc}
	    }
			\subfigure[COIL dataset]
	    {
	        \includegraphics[width=0.80\linewidth]{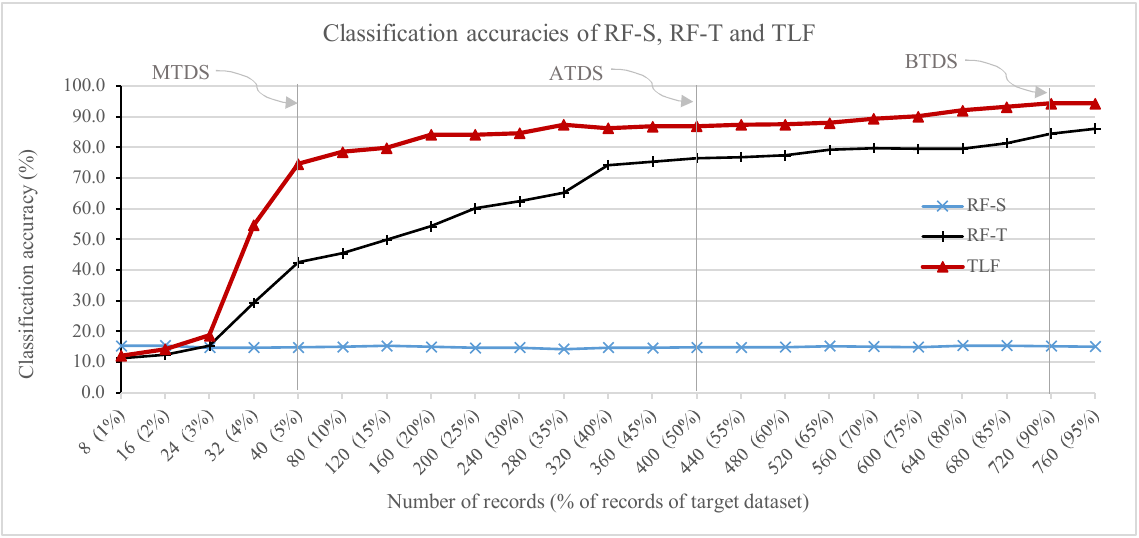}
	        \label{fig:justifytgtratio_coil}
	    }
	    \caption{Analysis of the effects of target dataset size on the performance of TLF, RF-S and RF-T.}
	    \label{fig:justifytgtratio}
\end{figure}

From the empirical evaluation on VLSC and COIL datasets (see Fig.~\ref{fig:justifytgtratio}a and Fig.~\ref{fig:justifytgtratio}b), we have several interesting observations. (1) The performances of both TLF and RF-T are improved with the increase of the size of $D^{new}_t$. This observation is sensible and natural. (2) The performance of RF-S is not improved even if the size of $D^{new}_t$ is increased. This observation is also reasonable as source and target data are drawn from different distributions, and the classifier is built only on $D_s$. (3) When the size of $D^{new}_t$ is less than 5\%, none of the methods perform well due to insufficient data to build a good classifier. (4) When the size of $D^{new}_t$ is 5\%, TLF performs better than RF-T indicating the minimum size of $D^{new}_t$ which is required to bridge the gap between the source and target domains. Moreover, the influence of source data is demonstrated in the performance of TLF. (5) When the size of $D^{new}_t$ is 50\%, both TLF and RF-T perform average indicating $D^{new}_t$ is still insufficient to build a good classifier. (6) When the size of $D^{new}_t$ is 90\% or above, both TLF and RF-T perform their best indicating that $D^{new}_t$ is sufficient to build a good forest. However, due to transferring knowledge from the source domain, TLF outperforms RF-S and RF-T.

All the observations verify the impact of the target dataset size on the performances of both transfer learning and non-transfer learning methods, and clearly demonstrate the need for transfer learning across domains. Although, TLF assumes a small labeled data in $D^{new}_t$, we evaluate and compare TLF and existing methods by creating $D^{new}_t$ and $D^{test}_t$ based on the following three scenarios:

\textbf{Minimun target dataset size (MTDS):} For each source$\mapsto$target pair, we create $D^{new}_t$ and $D^{test}_t$ having 5\% and 95\% of $D^{ori}_t$, respectively. 

\textbf{Average target dataset size (ATDS):} For each source$\mapsto$target pair, we create $D^{new}_t$ and $D^{test}_t$ having 50\% and 50\% of $D^{ori}_t$, respectively. 

\textbf{Best target dataset size (BTDS):} For each source$\mapsto$target pair, we create $D^{new}_t$ and $D^{test}_t$ having 90\% and 10\% of $D^{ori}_t$, respectively. 

In Fig.~\ref{fig:justifytgtratio}, the scenarios (i.e. MTDS, ATDS, and BTDS) are marked with vertical lines. We evaluate and compare TLF and existing methods for each scenario in the following section.

\subsection{Experimental Results}
\label{experimentalresults}

\subsubsection{Minimun target dataset size (MTDS) Results}
\label{experimental_mtds}
In MTDS scenario, we first create (new) target and test datasets having 5\% and 95\% records, respectively of the original target dataset (discussed in Section~\ref{impact_targetdata}), and then evaluate and compare the performances of TLF and existing methods. We present the average performances of 10 runs of the methods on 20 pairs of SDD and TDD that are created from PIE dataset in Table~\ref{tab:pie_detail_accuracy}. 

For each method with suffix '-S', we first build a classifier by applying the method on the SDD. We then calculate the classification accuracy by applying the classifier on the test data. For example, for combination 1 (``PIE05$\mapsto$PIE07'') the classification accuracy of RF-S is 5.2$\pm$1.8 which is reported in the first row (under the RF-S column) in Table~\ref{tab:pie_detail_accuracy}. Similarly, for each method with suffix '-T', we first build a classifier by applying the method on the TDD and then calculate the classification accuracy by applying the classifier on the test data.

For each transfer learning method, we take both SDD and TDD as input of the method and get a classifier as output. We then calculate the classification accuracy by applying the classifier on the test data. For example, for combination 1 (``PIE05$\mapsto$PIE07
'') the classification accuracy of TLF is 33.8$\pm$0.7 which is reported in the first row (under the TLF column) in Table~\ref{tab:pie_detail_accuracy}. The best results are highlighted in bold. TLF performs the best in all combinations.

For each method, we also calculate the average classification accuracy from the classification accuracies of 20 combinations as shown in the last row of the table. The average classification accuracy of TLF is 47.2$\pm$0.8. It is clear from the experimental results that TLF outperforms other techniques on PIE dataset.

\begin{table*}[!ht]
\tiny
	\centering
	\setlength{\belowcaptionskip}{-10pt}
	\setlength{\abovecaptionskip}{-5pt}	
	\caption{Average classification accuracy with standard deviation (\%) of TLF and existing methods on PIE dataset in MTDS scenario. The best results are highlighted in bold.}
	\renewcommand\arraystretch{1.1} 
	\renewcommand\tabcolsep{0pt} 
	\begin{tabular*}{\textwidth}{@{\extracolsep{\fill}}llccccccccccccccccc}
		\toprule
		\multicolumn{2}{c}{Domains} & \multicolumn{4}{c}{Baseline methods}& \multicolumn{12}{c}{Existing transfer learning methods}& \multirow{2}*{TLF}\\  
	\cline{1-2}\cline{3-6}\cline{7-18}
		Source&Target&RF-S&RF-T&SysFor-S&SysFor-T&EasyTL&ARTL&GFK&TCA&JDA&CORAL&BDA&MEDA&TOHAN&TNT&SHOT++&SHDA&\\
		\midrule
PIE05& PIE07&5.2$\pm$1.8&6.4$\pm$1.4&2.5$\pm$1.3&4.5$\pm$1.3&9.8$\pm$2.2&26.5$\pm$2.1&23.3$\pm$2.3&27.4$\pm$0.6&21.4$\pm$0.8&11.1$\pm$2.8&30.6$\pm$2.0&26.9$\pm$2.5&29.4$\pm$2.5&29.3$\pm$2.5&29.6$\pm$1.4&29.8$\pm$2.3&\textbf{33.8$\pm$0.7}\\
PIE05& PIE09&4.1$\pm$2.2&5.0$\pm$1.8&2.3$\pm$2.6&5.1$\pm$1.0&12.4$\pm$1.6&37.8$\pm$2.3&29.4$\pm$2.3&32.2$\pm$1.3&26.5$\pm$1.2&12.7$\pm$2.2&33.5$\pm$2.3&28.2$\pm$0.9&30.1$\pm$0.7&32.1$\pm$2.7&33.1$\pm$1.1&22.4$\pm$0.4&\textbf{38.1$\pm$1.5}\\
PIE05& PIE27&5.6$\pm$1.2&9.0$\pm$2.0&2.4$\pm$0.4&8.7$\pm$1.3&33.7$\pm$0.9&54.0$\pm$1.0&35.4$\pm$2.1&49.3$\pm$3.0&23.6$\pm$1.0&36.8$\pm$2.7&51.0$\pm$1.8&48.5$\pm$1.9&41.2$\pm$2.3&51.2$\pm$1.7&53.1$\pm$1.1&34.3$\pm$1.4&\textbf{58.4$\pm$0.5}\\
PIE05& PIE29&3.9$\pm$1.7&6.0$\pm$2.7&1.5$\pm$1.8&4.8$\pm$2.3&9.2$\pm$1.7&17.2$\pm$1.5&21.9$\pm$1.4&24.2$\pm$1.8&21.4$\pm$0.9&9.7$\pm$0.5&24.6$\pm$0.2&17.4$\pm$1.8&21.4$\pm$2.0&23.1$\pm$0.3&24.1$\pm$2.4&22.4$\pm$2.6&\textbf{24.9$\pm$1.2}\\
PIE07& PIE05&6.6$\pm$1.1&15.2$\pm$2.1&2.2$\pm$0.8&8.5$\pm$1.2&22.6$\pm$2.8&32.1$\pm$2.7&23.4$\pm$0.8&23.2$\pm$0.9&26.5$\pm$1.1&22.5$\pm$0.2&23.8$\pm$0.8&21.9$\pm$0.7&32.1$\pm$2.8&32.5$\pm$1.6&33.2$\pm$2.8&25.3$\pm$0.0&\textbf{38.6$\pm$0.5}\\
PIE07& PIE09&12.1$\pm$1.8&5.0$\pm$2.0&3.7$\pm$1.4&5.1$\pm$1.4&13.3$\pm$0.3&49.2$\pm$0.5&48.8$\pm$1.9&34.0$\pm$3.0&26.5$\pm$1.2&13.7$\pm$0.3&49.4$\pm$0.7&42.7$\pm$0.1&45.1$\pm$0.6&47.2$\pm$0.7&51.4$\pm$0.9&22.4$\pm$1.9&\textbf{52.6$\pm$0.7}\\
PIE07& PIE27&18.1$\pm$2.9&9.0$\pm$2.2&8.0$\pm$1.3&8.7$\pm$2.8&35.8$\pm$1.1&48.2$\pm$0.6&61.2$\pm$1.4&51.4$\pm$0.9&37.7$\pm$0.8&35.4$\pm$2.1&36.2$\pm$2.4&48.4$\pm$0.8&58.1$\pm$1.9&59.4$\pm$2.7&62.1$\pm$0.8&34.3$\pm$2.3&\textbf{66.2$\pm$0.1}\\
PIE07& PIE29&8.6$\pm$2.6&6.0$\pm$1.8&2.2$\pm$1.2&4.8$\pm$1.8&10.5$\pm$0.0&32.7$\pm$1.4&31.2$\pm$2.4&29.5$\pm$2.5&21.5$\pm$1.5&10.8$\pm$1.1&33.8$\pm$1.1&27.1$\pm$1.2&31.1$\pm$1.0&30.2$\pm$3.0&32.1$\pm$2.1&22.4$\pm$1.0&\textbf{36.4$\pm$0.2}\\
PIE09& PIE05&4.9$\pm$1.8&15.2$\pm$2.0&4.6$\pm$2.2&8.5$\pm$1.6&21.8$\pm$2.9&32.3$\pm$1.9&35.7$\pm$1.8&29.7$\pm$2.8&26.5$\pm$0.3&22.8$\pm$1.6&38.9$\pm$0.4&22.4$\pm$1.5&32.1$\pm$0.9&30.5$\pm$1.0&34.5$\pm$2.1&35.3$\pm$1.6&\textbf{39.2$\pm$0.1}\\
PIE09& PIE07&9.5$\pm$2.6&6.4$\pm$3.0&5.1$\pm$1.2&4.5$\pm$2.7&12.3$\pm$1.4&42.8$\pm$1.9&25.3$\pm$2.5&36.6$\pm$1.5&35.6$\pm$2.4&12.8$\pm$1.9&25.2$\pm$0.3&43.4$\pm$2.7&41.2$\pm$0.2&42.2$\pm$2.2&45.1$\pm$1.9&29.8$\pm$0.4&\textbf{52.9$\pm$0.1}\\
PIE09& PIE27&14.7$\pm$2.0&9.0$\pm$2.7&10.4$\pm$1.5&8.7$\pm$1.8&35.6$\pm$2.3&53.3$\pm$2.9&62.8$\pm$1.3&54.7$\pm$1.2&45.1$\pm$2.5&36.0$\pm$0.8&36.6$\pm$0.7&53.6$\pm$1.8&61.2$\pm$1.4&58.2$\pm$0.6&60.1$\pm$2.9&34.3$\pm$0.0&\textbf{68.2$\pm$1.9}\\
PIE09& PIE29&7.3$\pm$2.7&6.0$\pm$2.7&5.4$\pm$2.5&4.8$\pm$2.8&12.7$\pm$1.0&32.5$\pm$1.4&39.2$\pm$0.1&36.1$\pm$0.7&31.5$\pm$2.9&12.6$\pm$2.3&36.2$\pm$0.8&33.1$\pm$1.6&25.4$\pm$1.0&30.1$\pm$1.2&34.1$\pm$2.9&22.4$\pm$2.2&\textbf{41.1$\pm$1.2}\\
PIE27& PIE05&7.7$\pm$2.3&15.2$\pm$1.3&2.5$\pm$1.9&8.5$\pm$2.9&36.1$\pm$1.5&53.8$\pm$0.0&57.6$\pm$0.7&25.3$\pm$0.3&52.6$\pm$0.9&38.6$\pm$1.5&35.5$\pm$1.7&48.0$\pm$0.2&51.0$\pm$0.6&54.1$\pm$1.9&56.2$\pm$2.7&43.0$\pm$0.3&\textbf{60.5$\pm$0.1}\\
PIE27& PIE07&15.6$\pm$1.5&6.4$\pm$1.0&5.7$\pm$2.1&4.5$\pm$2.2&19.2$\pm$2.9&59.0$\pm$0.4&46.6$\pm$1.6&56.0$\pm$0.1&65.2$\pm$0.7&18.9$\pm$0.7&70.1$\pm$0.1&59.1$\pm$1.8&67.1$\pm$1.4&65.1$\pm$1.9&68.1$\pm$3.0&29.8$\pm$1.8&\textbf{74.0$\pm$0.2}\\
PIE27& PIE09&18.8$\pm$2.8&5.0$\pm$2.5&8.2$\pm$1.4&5.1$\pm$1.1&21.1$\pm$2.2&66.6$\pm$0.8&73.4$\pm$1.3&62.5$\pm$2.8&45.7$\pm$1.9&21.8$\pm$0.6&47.5$\pm$1.6&66.4$\pm$2.8&52.4$\pm$1.7&63.5$\pm$1.5&61.5$\pm$0.3&22.4$\pm$0.3&\textbf{76.9$\pm$1.8}\\
PIE27& PIE29&9.8$\pm$1.2&6.0$\pm$1.8&4.3$\pm$1.0&4.8$\pm$1.5&14.9$\pm$0.1&43.8$\pm$2.8&24.8$\pm$0.6&44.9$\pm$1.9&36.5$\pm$0.1&15.1$\pm$2.2&42.2$\pm$3.0&37.4$\pm$0.7&35.2$\pm$2.4&34.1$\pm$0.6&39.4$\pm$0.6&22.4$\pm$1.6&\textbf{48.1$\pm$1.4}\\
PIE29& PIE05&4.6$\pm$1.2&15.2$\pm$1.8&2.0$\pm$1.5&8.5$\pm$2.8&14.8$\pm$1.5&21.2$\pm$2.6&25.1$\pm$1.5&20.2$\pm$0.5&21.5$\pm$2.5&18.7$\pm$1.4&25.6$\pm$1.0&12.6$\pm$1.3&21.4$\pm$1.8&25.1$\pm$2.6&24.1$\pm$2.1&24.4$\pm$1.6&\textbf{29.2$\pm$1.5}\\
PIE29& PIE07&5.8$\pm$2.8&6.4$\pm$1.5&2.1$\pm$2.3&4.5$\pm$2.7&9.1$\pm$1.9&25.3$\pm$2.3&29.6$\pm$1.7&25.6$\pm$0.2&25.1$\pm$0.4&9.9$\pm$0.8&28.4$\pm$0.5&25.7$\pm$0.5&20.1$\pm$2.8&24.1$\pm$2.3&23.5$\pm$0.3&23.0$\pm$3.0&\textbf{30.2$\pm$0.3}\\
PIE29& PIE09&3.4$\pm$1.4&5.0$\pm$2.1&2.2$\pm$2.7&5.1$\pm$2.0&10.8$\pm$1.3&28.0$\pm$0.8&23.3$\pm$2.9&28.6$\pm$1.8&21.4$\pm$2.2&12.3$\pm$2.7&23.2$\pm$1.3&28.4$\pm$2.1&21.4$\pm$0.3&23.3$\pm$2.6&23.4$\pm$0.5&22.4$\pm$1.4&\textbf{31.2$\pm$1.1}\\
PIE29& PIE27&5.0$\pm$2.6&9.0$\pm$1.2&2.2$\pm$1.2&8.7$\pm$2.2&26.1$\pm$0.3&31.2$\pm$2.9&39.1$\pm$2.8&35.9$\pm$0.3&32.2$\pm$2.0&27.0$\pm$1.2&43.2$\pm$1.2&31.9$\pm$0.6&25.7$\pm$0.4&28.1$\pm$0.6&34.9$\pm$1.7&34.3$\pm$1.1&\textbf{43.7$\pm$0.5}\\
\midrule
\multicolumn{2}{c}{Average}&8.6$\pm$2.0&8.3$\pm$2.0&4.0$\pm$1.6&6.3$\pm$2.0&19.1$\pm$1.5&39.4$\pm$1.7&37.9$\pm$1.7&36.4$\pm$1.4&32.2$\pm$1.4&20.0$\pm$1.5&36.8$\pm$1.2&36.2$\pm$1.4&37.1$\pm$1.4&39.2$\pm$1.7&41.2$\pm$1.7&27.8$\pm$1.4&\textbf{47.2$\pm$0.8}\\
		\bottomrule
	\end{tabular*}
	\label{tab:pie_detail_accuracy}
\end{table*}

Table~\ref{tab:oc_detail_accuracy} presents the performances of TLF and existing methods on the OC dataset. The best results are highlighted in bold. TLF performs the best in all pairs. The overall average classification accuracies also clearly indicate that TLF outperforms other methods on OC dataset. 

\begin{table*}[!ht]
\tiny
	\centering
	\setlength{\belowcaptionskip}{-10pt}
	\setlength{\abovecaptionskip}{-5pt}	
	\caption{Average classification accuracy with standard deviation (\%) of TLF and existing methods on OC dataset in MTDS scenario. The best results are highlighted in bold.}
	\renewcommand\arraystretch{1.1} 
	\renewcommand\tabcolsep{0pt} 
	\begin{tabular*}{\textwidth}{@{\extracolsep{\fill}}llccccccccccccccccc}
		\toprule
		\multicolumn{2}{c}{Domains} & \multicolumn{4}{c}{Baseline methods}& \multicolumn{12}{c}{Existing transfer learning methods}& \multirow{2}*{TLF}\\  
	\cline{1-2}\cline{3-6}\cline{7-18}
		Source&Target&RF-S&RF-T&SysFor-S&SysFor-T&EasyTL&ARTL&GFK&TCA&JDA&CORAL&BDA&MEDA&TOHAN&TNT&SHOT++&SHDA&\\
		\midrule
amazon& Caltech10&15.3$\pm$1.4&20.7$\pm$1.7&13.4$\pm$2.4&22.9$\pm$1.6&33.1$\pm$1.9&26.8$\pm$2.9&23.3$\pm$2.6&26.4$\pm$1.7&21.4$\pm$1.4&24.9$\pm$0.1&26.3$\pm$2.1&21.9$\pm$0.2&23.4$\pm$0.5&26.4$\pm$0.0&27.5$\pm$2.3&27.7$\pm$2.1&\textbf{35.4$\pm$1.9}\\
amazon& dslr&10.7$\pm$2.1&8.7$\pm$1.6&6.0$\pm$2.9&6.7$\pm$1.5&4.8$\pm$0.9&29.5$\pm$1.3&28.0$\pm$2.6&30.9$\pm$2.3&24.2$\pm$2.8&10.5$\pm$1.3&24.2$\pm$0.8&22.8$\pm$0.3&20.1$\pm$2.1&24.1$\pm$1.3&26.1$\pm$2.5&28.2$\pm$2.1&\textbf{31.3$\pm$1.4}\\
amazon& webcam&13.2$\pm$3.0&23.2$\pm$1.1&10.4$\pm$2.9&18.6$\pm$2.3&9.3$\pm$2.0&28.6$\pm$1.7&29.4$\pm$0.3&30.7$\pm$1.3&29.6$\pm$1.7&8.5$\pm$1.2&32.1$\pm$2.8&27.9$\pm$1.2&21.5$\pm$0.1&25.4$\pm$1.4&25.8$\pm$2.5&26.8$\pm$0.5&\textbf{32.5$\pm$1.0}\\
Caltech10&amazon&20.7$\pm$1.2&25.4$\pm$1.7&9.2$\pm$2.4&27.3$\pm$1.0&24.8$\pm$0.1&11.7$\pm$1.2&23.7$\pm$2.3&16.5$\pm$2.3&17.0$\pm$1.2&18.0$\pm$2.3&17.5$\pm$0.8&31.7$\pm$1.7&28.3$\pm$1.8&30.2$\pm$1.8&30.1$\pm$1.1&36.0$\pm$2.8&\textbf{36.7$\pm$0.0}\\
Caltech10& dslr&10.7$\pm$1.0&8.7$\pm$2.3&6.0$\pm$2.3&6.7$\pm$2.5&20.4$\pm$2.8&24.8$\pm$1.1&26.9$\pm$0.9&28.9$\pm$0.9&21.5$\pm$0.7&20.4$\pm$2.0&28.2$\pm$2.9&28.8$\pm$2.8&21.4$\pm$1.7&28.4$\pm$2.6&30.3$\pm$0.2&29.5$\pm$1.1&\textbf{36.9$\pm$0.0}\\
Caltech10& webcam&12.9$\pm$1.4&23.2$\pm$2.2&7.1$\pm$2.2&18.6$\pm$2.9&9.0$\pm$2.1&24.6$\pm$1.9&22.9$\pm$2.8&27.5$\pm$1.3&22.1$\pm$0.3&20.5$\pm$2.0&25.0$\pm$2.2&22.9$\pm$2.7&24.7$\pm$2.4&25.4$\pm$2.3&27.4$\pm$1.1&29.3$\pm$2.9&\textbf{32.9$\pm$0.8}\\
dslr&amazon&19.7$\pm$1.7&25.4$\pm$2.5&9.2$\pm$2.8&27.3$\pm$1.6&31.3$\pm$2.4&27.1$\pm$0.3&22.3$\pm$1.5&28.2$\pm$2.5&24.7$\pm$0.6&32.3$\pm$0.2&13.5$\pm$2.9&23.7$\pm$0.3&27.4$\pm$0.0&29.7$\pm$2.1&30.5$\pm$1.6&19.3$\pm$2.4&\textbf{35.4$\pm$0.9}\\
dslr& Caltech10&16.9$\pm$2.3&20.7$\pm$2.4&10.2$\pm$1.2&22.9$\pm$1.3&36.8$\pm$2.4&28.8$\pm$2.8&22.9$\pm$1.3&28.4$\pm$2.4&24.0$\pm$2.1&40.3$\pm$1.2&26.7$\pm$0.0&23.8$\pm$1.2&31.4$\pm$1.3&34.1$\pm$1.2&38.7$\pm$2.1&21.7$\pm$1.0&\textbf{42.9$\pm$1.0}\\
dslr& webcam&19.6$\pm$1.1&23.2$\pm$1.2&13.6$\pm$1.4&18.6$\pm$1.7&16.0$\pm$1.1&64.6$\pm$1.5&46.6$\pm$0.0&46.6$\pm$1.4&47.9$\pm$1.8&18.0$\pm$2.1&68.2$\pm$1.1&64.3$\pm$2.5&54.7$\pm$0.8&55.1$\pm$1.5&60.1$\pm$1.4&33.6$\pm$1.6&\textbf{68.6$\pm$1.7}\\
webcam&amazon&14.9$\pm$1.3&25.4$\pm$2.2&9.2$\pm$1.3&27.3$\pm$1.4&32.8$\pm$0.9&17.5$\pm$2.3&23.6$\pm$0.4&11.2$\pm$1.5&19.5$\pm$2.0&32.9$\pm$2.3&23.2$\pm$2.6&20.9$\pm$0.1&24.1$\pm$1.7&25.4$\pm$0.3&27.8$\pm$1.2&21.9$\pm$1.0&\textbf{35.8$\pm$1.2}\\
webcam& Caltech10&15.3$\pm$2.6&20.7$\pm$1.4&11.6$\pm$2.6&22.9$\pm$1.2&23.0$\pm$0.3&24.0$\pm$0.8&21.9$\pm$1.9&21.3$\pm$1.0&18.8$\pm$1.0&23.4$\pm$0.1&24.3$\pm$1.1&18.9$\pm$2.2&20.1$\pm$2.8&24.1$\pm$0.7&26.1$\pm$2.7&21.8$\pm$0.2&\textbf{29.4$\pm$1.5}\\
webcam& dslr&9.4$\pm$2.8&8.7$\pm$1.8&9.4$\pm$2.0&6.7$\pm$1.2&18.6$\pm$0.3&59.1$\pm$2.9&47.5$\pm$0.9&49.0$\pm$2.3&54.4$\pm$0.9&10.8$\pm$1.1&55.7$\pm$1.8&55.0$\pm$2.1&54.3$\pm$1.8&55.6$\pm$1.5&61.0$\pm$1.6&50.3$\pm$1.2&\textbf{62.7$\pm$1.5}\\
\midrule
\multicolumn{2}{c}{Average}&14.9$\pm$1.8&19.5$\pm$1.8&9.6$\pm$2.2&18.9$\pm$1.7&21.6$\pm$1.4&30.6$\pm$1.7&28.3$\pm$1.4&28.8$\pm$1.7&27.1$\pm$1.4&21.7$\pm$1.3&30.4$\pm$1.8&30.2$\pm$1.4&29.3$\pm$1.4&32.0$\pm$1.4&34.3$\pm$1.7&28.9$\pm$1.6&\textbf{40.0$\pm$1.1}\\
		\bottomrule
	\end{tabular*}
	\label{tab:oc_detail_accuracy}
\end{table*}

Similar to PIE and OC datasets, for the remaining five datasets TLF achieves higher classification accuracies than existing techniques. Due to space limitation, we present the average classification accuracies (see Table~\ref{tab:overall_accuracy}) of the TLF and existing techniques on all datasets. Bold values in the tables indicate the best results. From Table~\ref{tab:overall_accuracy}, we can see that TLF outperforms the other techniques in terms of overall average classification accuracy for all datasets in MTDS scenario.

\begin{table*}[!ht]
\tiny
	\centering
	\setlength{\belowcaptionskip}{-10pt}
	\setlength{\abovecaptionskip}{-5pt}	
	\caption{Overall average classification accuracy with standard deviation (\%) of TLF and existing methods on all datasets in MTDS scenario. The best results are highlighted in bold.}
	\renewcommand\arraystretch{1.1} 
	\renewcommand\tabcolsep{0pt} 
	\begin{tabular*}{\textwidth}{@{\extracolsep{\fill}}lccccccccccccccccc}
		\toprule
		\multirow{2}*{Dataset}& \multicolumn{4}{c}{Baseline methods}& \multicolumn{12}{c}{Existing transfer learning methods}& \multirow{2}*{TLF}\\  
	\cline{2-5}\cline{6-17}
		&RF-S&RF-T&SysFor-S&SysFor-T&EasyTL&ARTL&GFK&TCA&JDA&CORAL&BDA&MEDA&TOHAN&TNT&SHOT++&SHDA&\\
		\midrule
PIE&8.6$\pm$2.0&8.3$\pm$2.0&4.0$\pm$1.6&6.3$\pm$2.0&19.1$\pm$1.5&39.4$\pm$1.7&37.9$\pm$1.7&36.4$\pm$1.4&32.2$\pm$1.4&20.0$\pm$1.5&36.8$\pm$1.2&36.2$\pm$1.4&37.1$\pm$1.4&39.2$\pm$1.7&41.2$\pm$1.7&27.8$\pm$1.4&\textbf{47.2$\pm$0.8}\\
OC&14.9$\pm$1.8&19.5$\pm$1.8&9.6$\pm$2.2&18.9$\pm$1.7&21.6$\pm$1.4&30.6$\pm$1.7&28.3$\pm$1.4&28.8$\pm$1.7&27.1$\pm$1.4&21.7$\pm$1.3&30.4$\pm$1.8&30.2$\pm$1.4&29.3$\pm$1.4&32.0$\pm$1.4&34.3$\pm$1.7&28.9$\pm$1.6&\textbf{40.0$\pm$1.1}\\
NG&37.6$\pm$2.5&45.8$\pm$2.2&35.0$\pm$1.7&43.4$\pm$2.3&43.8$\pm$1.1&58.9$\pm$0.9&56.0$\pm$1.3&54.7$\pm$1.7&55.4$\pm$1.3&43.8$\pm$1.9&56.5$\pm$1.1&55.7$\pm$0.6&55.4$\pm$2.3&58.6$\pm$0.9&59.5$\pm$1.6&56.0$\pm$1.9&\textbf{64.0$\pm$0.9}\\
VLSC&22.2$\pm$2.0&23.5$\pm$1.9&20.6$\pm$2.1&22.1$\pm$1.8&31.3$\pm$1.1&62.9$\pm$1.5&32.1$\pm$1.4&28.0$\pm$1.5&29.3$\pm$1.9&27.4$\pm$1.4&34.8$\pm$1.1&55.6$\pm$1.4&48.0$\pm$1.6&53.7$\pm$1.7&58.0$\pm$1.5&47.7$\pm$1.1&\textbf{65.1$\pm$1.0}\\
RT&33.2$\pm$2.0&38.5$\pm$2.3&32.4$\pm$1.7&37.5$\pm$2.3&45.0$\pm$0.9&56.3$\pm$0.8&46.3$\pm$1.2&50.1$\pm$1.5&50.0$\pm$1.9&48.5$\pm$1.9&47.7$\pm$1.3&53.8$\pm$0.9&51.2$\pm$1.3&53.5$\pm$1.4&57.2$\pm$1.9&51.7$\pm$2.1&\textbf{61.5$\pm$1.1}\\
MU&13.8$\pm$2.0&21.2$\pm$2.1&11.1$\pm$2.0&18.6$\pm$2.1&37.4$\pm$1.7&49.1$\pm$1.8&33.8$\pm$1.7&36.8$\pm$1.7&36.3$\pm$1.3&36.9$\pm$1.8&35.9$\pm$1.7&47.0$\pm$1.4&42.0$\pm$1.3&47.4$\pm$1.4&48.9$\pm$1.5&40.6$\pm$1.4&\textbf{56.8$\pm$1.1}\\
COIL&14.8$\pm$1.7&41.4$\pm$1.9&13.9$\pm$2.0&33.4$\pm$2.1&54.5$\pm$1.7&68.3$\pm$1.6&62.1$\pm$1.9&62.8$\pm$1.9&60.9$\pm$1.3&54.1$\pm$1.4&62.4$\pm$1.3&64.5$\pm$1.3&57.0$\pm$1.5&60.9$\pm$1.5&66.8$\pm$1.6&63.1$\pm$1.6&\textbf{74.6$\pm$1.2}\\
\midrule
Average&20.7$\pm$2.0&28.3$\pm$2.0&18.1$\pm$1.9&25.8$\pm$2.0&36.1$\pm$1.3&52.2$\pm$1.4&42.3$\pm$1.5&42.5$\pm$1.6&41.6$\pm$1.5&36.0$\pm$1.6&43.5$\pm$1.3&49.0$\pm$1.2&45.7$\pm$1.5&49.3$\pm$1.4&52.3$\pm$1.6&45.1$\pm$1.6&\textbf{58.5$\pm$1.0}\\
		\bottomrule
	\end{tabular*}
	\label{tab:overall_accuracy}
\end{table*}

\subsubsection{Average target dataset size (ATDS) Results}
\label{experimental_atds}
In ATDS scenario, we evaluate and compare the performances of TLF and existing methods by creating (new) target and test datasets having 50\% and 50\% records, respectively of the original target dataset(discussed in Section~\ref{impact_targetdata}). For all source and target pairs, TLF achieves higher classification accuracies than the existing techniques on all datasets. Due to space limitation, in Table~\ref{tab:overall_accuracy_atds} we only present the overall average classification accuracies of the TLF and existing techniques on all datasets. The best results are highlighted in bold. From Table~\ref{tab:overall_accuracy_atds}, we can see that all methods perform better in ATDS scenario comparing with the in MTDS, and TLF achieves the highest overall average classification accuracies for all datasets. 

\begin{table*}[!ht]
\tiny
	\centering
	\setlength{\belowcaptionskip}{-10pt}
	\setlength{\abovecaptionskip}{-5pt}	
	\caption{Overall average classification accuracy with standard deviation (\%) of TLF and existing methods on all datasets for ATDS. The best results are highlighted in bold.}
	\renewcommand\arraystretch{1.1} 
	\renewcommand\tabcolsep{0pt} 
	\begin{tabular*}{\textwidth}{@{\extracolsep{\fill}}lccccccccccccccccc}
		\toprule
		\multirow{2}*{Dataset}& \multicolumn{4}{c}{Baseline methods}& \multicolumn{12}{c}{Existing transfer learning methods}& \multirow{2}*{TLF}\\  
	\cline{2-5}\cline{6-17}
		&RF-S&RF-T&SysFor-S&SysFor-T&EasyTL&ARTL&GFK&TCA&JDA&CORAL&BDA&MEDA&TOHAN&TNT&SHOT++&SHDA&\\
		\midrule
PIE&8.6$\pm$2.0&40.2$\pm$1.9&4.1$\pm$1.4&41.0$\pm$1.7&42.6$\pm$0.9&59.4$\pm$0.1&53.6$\pm$1.8&58.4$\pm$1.2&56.4$\pm$0.4&41.3$\pm$1.9&53.2$\pm$0.4&51.5$\pm$1.5&53.6$\pm$2.4&59.4$\pm$0.3&62.5$\pm$2.4&58.7$\pm$3.0&\textbf{64.4$\pm$0.5}\\
OC&15.4$\pm$2.9&39.2$\pm$3.0&10.6$\pm$2.7&40.8$\pm$2.3&41.2$\pm$1.8&61.5$\pm$0.8&57.2$\pm$0.4&48.5$\pm$0.0&50.1$\pm$1.7&48.3$\pm$0.4&52.9$\pm$2.0&50.4$\pm$1.5&49.6$\pm$0.2&57.0$\pm$2.6&59.5$\pm$0.8&57.7$\pm$0.9&\textbf{62.5$\pm$1.4}\\
NG&38.6$\pm$2.0&58.3$\pm$3.0&35.0$\pm$2.7&58.0$\pm$2.6&65.9$\pm$1.9&69.6$\pm$2.6&75.4$\pm$1.3&74.5$\pm$0.3&79.6$\pm$0.0&64.5$\pm$2.1&71.3$\pm$1.6&69.4$\pm$2.0&75.6$\pm$0.6&75.5$\pm$0.2&78.6$\pm$1.8&71.5$\pm$1.5&\textbf{81.7$\pm$0.6}\\
VLSC&23.2$\pm$2.9&57.5$\pm$2.8&21.3$\pm$1.2&57.1$\pm$2.0&51.3$\pm$1.6&61.2$\pm$2.5&62.7$\pm$1.1&53.6$\pm$2.4&57.6$\pm$2.3&51.3$\pm$0.9&57.2$\pm$1.7&64.3$\pm$3.0&62.1$\pm$1.2&62.6$\pm$2.8&67.3$\pm$0.7&53.6$\pm$2.5&\textbf{68.5$\pm$1.4}\\
RT&33.1$\pm$1.8&59.6$\pm$1.6&33.5$\pm$1.9&54.4$\pm$1.8&62.5$\pm$1.5&68.5$\pm$2.0&69.5$\pm$2.3&71.2$\pm$2.4&68.2$\pm$1.9&65.6$\pm$0.4&57.2$\pm$1.7&68.7$\pm$0.8&68.7$\pm$0.8&71.2$\pm$2.7&74.4$\pm$0.8&69.5$\pm$1.8&\textbf{79.5$\pm$0.8}\\
MU&14.5$\pm$2.9&57.3$\pm$2.8&12.3$\pm$2.7&57.7$\pm$1.1&61.9$\pm$0.2&61.3$\pm$1.0&54.9$\pm$2.3&54.0$\pm$1.8&56.7$\pm$0.7&60.5$\pm$0.6&52.7$\pm$0.7&61.4$\pm$0.1&62.6$\pm$0.6&68.2$\pm$1.1&72.0$\pm$0.8&67.5$\pm$2.0&\textbf{78.7$\pm$1.6}\\
COIL&14.9$\pm$1.5&76.5$\pm$2.7&14.1$\pm$2.3&72.3$\pm$2.7&71.6$\pm$2.6&74.2$\pm$2.9&81.3$\pm$1.4&72.8$\pm$1.2&70.2$\pm$2.6&68.3$\pm$1.5&72.3$\pm$1.5&75.5$\pm$0.5&69.3$\pm$0.5&74.6$\pm$0.6&81.3$\pm$1.3&79.7$\pm$0.5&\textbf{87.0$\pm$0.3}\\
\midrule
Average&21.2$\pm$2.3&55.5$\pm$2.5&18.7$\pm$2.1&54.5$\pm$2.0&56.7$\pm$1.5&65.1$\pm$1.7&64.9$\pm$1.5&61.9$\pm$1.3&62.7$\pm$1.4&57.1$\pm$1.1&59.5$\pm$1.4&63.0$\pm$1.3&63.1$\pm$0.9&66.9$\pm$1.5&70.8$\pm$1.2&65.5$\pm$1.7&\textbf{74.6$\pm$1.0}\\
		\bottomrule
	\end{tabular*}
	\label{tab:overall_accuracy_atds}
\end{table*}

\subsubsection{Best target dataset size (BTDS) Results}
\label{experimental_btds}
In BTDS scenario, we first create a (new) target dataset by randomly choosing 90\% records of the original target dataset and consider the remaining 10\% records for the test dataset. We then evaluate and compare the performances of TLF and existing methods. Similar to MTDS and ADTS scenarios, for all source and target pairs TLF achieves higher classification accuracies than the existing techniques on all datasets. Again, due to space limitation, we only present the overall average classification accuracies of the TLF and existing techniques on all datasets in Table~\ref{tab:overall_accuracy_btds}. The best results are highlighted in bold. From Table~\ref{tab:overall_accuracy_btds}, we can see that while all existing methods perform well due to sufficient data in TDD in BTDS scenario, TLF outperforms the existing methods in terms of overall average classification accuracy for all datasets. 
\begin{table*}[!ht]
\tiny
	\centering
	\setlength{\belowcaptionskip}{-10pt}
	\setlength{\abovecaptionskip}{-5pt}	
	\caption{Overall average classification accuracy with standard deviation (\%) of TLF and existing methods on all datasets for BTDS. The best results are highlighted in bold.}
	\renewcommand\arraystretch{1.1} 
	\renewcommand\tabcolsep{0pt} 
	\begin{tabular*}{\textwidth}{@{\extracolsep{\fill}}lccccccccccccccccc}
		\toprule
		\multirow{2}*{Dataset}& \multicolumn{4}{c}{Baseline methods}& \multicolumn{12}{c}{Existing transfer learning methods}& \multirow{2}*{TLF}\\  
	\cline{2-5}\cline{6-17}
		&RF-S&RF-T&SysFor-S&SysFor-T&EasyTL&ARTL&GFK&TCA&JDA&CORAL&BDA&MEDA&TOHAN&TNT&SHOT++&SHDA&\\
		\midrule
PIE&8.5$\pm$1.9&61.5$\pm$2.5&3.9$\pm$1.7&62.2$\pm$1.9&63.5$\pm$2.6&68.5$\pm$2.0&61.9$\pm$2.7&68.8$\pm$1.7&67.5$\pm$1.3&62.1$\pm$1.9&63.5$\pm$2.6&65.4$\pm$0.5&64.5$\pm$1.6&68.2$\pm$0.5&69.9$\pm$1.3&68.7$\pm$2.5&\textbf{71.6$\pm$0.4}\\
OC&16.2$\pm$2.6&60.2$\pm$1.2&11.2$\pm$2.9&59.3$\pm$2.5&60.2$\pm$1.7&67.5$\pm$2.0&60.5$\pm$2.1&64.5$\pm$0.8&64.8$\pm$0.8&61.2$\pm$2.6&60.6$\pm$0.0&67.5$\pm$0.0&63.6$\pm$1.5&69.8$\pm$0.8&69.8$\pm$0.8&67.5$\pm$1.5&\textbf{72.5$\pm$0.4}\\
NG&39.4$\pm$2.9&65.1$\pm$1.5&36.2$\pm$2.0&67.3$\pm$1.1&75.5$\pm$1.7&78.6$\pm$1.4&86.5$\pm$2.3&84.6$\pm$1.2&87.9$\pm$2.2&86.1$\pm$1.6&87.8$\pm$0.3&76.7$\pm$0.4&86.3$\pm$0.8&86.5$\pm$0.9&89.6$\pm$2.8&81.6$\pm$2.1&\textbf{92.3$\pm$1.1}\\
VLSC&22.7$\pm$2.5&66.5$\pm$1.8&22.3$\pm$1.0&66.7$\pm$2.2&67.3$\pm$0.4&71.1$\pm$1.0&68.3$\pm$2.6&68.2$\pm$2.3&68.5$\pm$0.9&65.6$\pm$1.6&68.3$\pm$0.5&70.2$\pm$2.3&68.6$\pm$0.1&72.3$\pm$1.8&73.2$\pm$1.0&71.7$\pm$2.0&\textbf{74.5$\pm$1.8}\\
RT&32.5$\pm$2.9&67.5$\pm$1.2&32.9$\pm$2.4&68.0$\pm$2.8&68.5$\pm$1.1&81.5$\pm$2.7&71.6$\pm$3.0&76.5$\pm$1.7&76.6$\pm$1.1&78.6$\pm$1.1&75.8$\pm$0.7&79.6$\pm$0.6&76.5$\pm$2.8&84.5$\pm$2.0&85.2$\pm$2.0&83.2$\pm$2.9&\textbf{87.3$\pm$1.7}\\
MU&15.3$\pm$2.1&74.2$\pm$1.9&12.0$\pm$2.7&74.6$\pm$2.1&75.3$\pm$2.3&84.5$\pm$2.9&74.6$\pm$2.2&81.2$\pm$2.0&76.4$\pm$0.5&80.7$\pm$1.4&78.6$\pm$2.2&81.7$\pm$2.0&81.0$\pm$0.9&85.6$\pm$2.9&86.6$\pm$1.1&81.2$\pm$0.4&\textbf{92.3$\pm$1.1}\\
COIL&15.2$\pm$2.8&84.5$\pm$2.3&15.2$\pm$1.2&85.2$\pm$2.0&86.5$\pm$0.5&90.0$\pm$1.2&86.6$\pm$1.1&86.9$\pm$1.8&87.5$\pm$1.9&87.9$\pm$1.1&89.7$\pm$2.1&87.7$\pm$0.4&89.7$\pm$3.0&90.0$\pm$3.0&91.5$\pm$2.7&90.1$\pm$1.2&\textbf{94.4$\pm$0.4}\\
\midrule
Average&21.4$\pm$2.5&68.5$\pm$1.8&19.1$\pm$2.0&69.0$\pm$2.1&71.0$\pm$1.5&77.4$\pm$1.9&72.8$\pm$2.3&75.8$\pm$1.6&75.6$\pm$1.2&74.6$\pm$1.6&74.9$\pm$1.2&75.5$\pm$0.9&75.7$\pm$1.5&79.6$\pm$1.7&80.8$\pm$1.7&77.7$\pm$1.8&\textbf{83.6$\pm$1.0}\\
		\bottomrule
	\end{tabular*}
	\label{tab:overall_accuracy_btds}
\end{table*}

From Table~\ref{tab:overall_accuracy}, Table~\ref{tab:overall_accuracy_atds} and Table~\ref{tab:overall_accuracy_btds}, we have the following observations. (1) The effectiveness of TLF is demonstrated in all scenarios since it consistently outperforms existing methods. (2) Only source domain data is not sufficient to build a good model for target domain which is different in terms of distributions, and features numbers and types. (3) For BTDS scenario, we see Table~\ref{tab:overall_accuracy_btds} from that the overall average accuracies of RF-T, SysFor-T and TLF are 68.5$\pm$1.8, 69.0$\pm$2.1, and 83.6$\pm$1.0, respectively. These results show the advantage of transferring knowledge from SDD and TDD over traditional ML methods even if the ML methods are learned with almost full data. (4) For all scenarios, the performance of SHOT++ is the second best, which demonstrates the effectiveness of deep neural network models in transfer learning.

\subsection{Statistical Analysis of the Results}
\label{Experimental_stat_analysis}

We now analyze the results by using a statistical non-parametric sign test~\cite{mason1994statistics} and Nemenyi~\cite{demvsar2006statistical} test for all datasets in all scenarios (i.e. MTDS, ATDS, and BTDS) to evaluate the statistical significance of the superiority of TLF over the existing methods. 

In MTDS scenario, we carry out sign test on the results of TLF with other methods (one by one) at the right-tailed by considering significance level $\alpha=0.025$ (i.e. 97.5\% significance level). While we compare TLF with an existing method (say Q) we obtain a z-value (test static value) which is used to determine the significant improvement of TLF over Q. The z-value for a comparison (i.e. TLF vs Q) on a dataset is calculated as follows~\cite{mason1994statistics}. 
{\small
\begin{equation}
\begin{aligned}
  z = \left\{
  \begin{array}{l l}
    \frac{x+\frac{1}{2} - \frac{n}{2}}{\frac{\sqrt{n}}{2}} & \text{ if $x > \frac{n}{2}$}\\		
		\frac{x-\frac{1}{2} - \frac{n}{2}}{\frac{\sqrt{n}}{2}}  & \text{otherwise}\\
  \end{array} \right.
\end{aligned}
\label{eq_z_signtest}
\end{equation}
} 
where $n$ is the total number of source and target pairs of a dataset, $x$ is the number of runs TLF achieves higher accuracy than Q. Besides, the z(ref) is the critical value which is 1.96 at $p<0.025$ and right-tailed~\cite{mason1994statistics}. The performance of TLF is significantly better than Q if z (test static) is greater than z(ref). 
We now see the z (test static) value for the TLF vs EasyTL comparison on NG dataset. The dataset has five pairs (discussed in Section~\ref{datapreparation}) that is $n=5$. TLF achieves higher accuracy than EasyTL in all cases resulting in $x=5$. According to Eq.~\eqref{eq_z_signtest}, the z value is 2.68 which is higher than the z(ref), 1.96. Therefore, TLF performs significantly better than EasyTL at $p<0.025$, and right-tailed. Similarly, we obtain z-values for other comparisons based on all datasets as shown in Fig.~\ref{fig:signtest}. The sign test results (see Fig.~\ref{fig:signtest}) indicate that both TLF performs significantly better than the other methods (at $z>1.96$, $p<0.025$, and right-tailed) on all datasets in MTDS scenario. It is worth mentioning that for ATDS and BTDS scenarios, we obtain similar results which are not presented here due to space limitation. 

\begin{figure*}[ht!]
\centering
  \setlength{\belowcaptionskip}{0pt}
	\setlength{\abovecaptionskip}{0pt}	
      \includegraphics[width=0.90\linewidth]{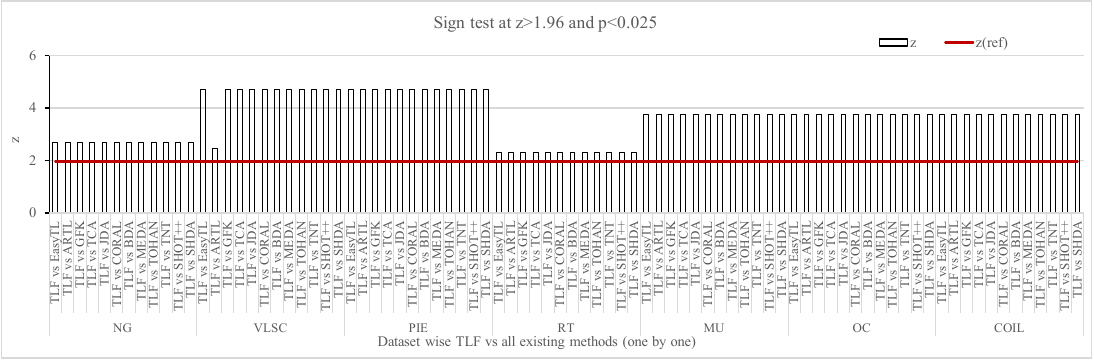}
	    \caption{Statistical significance analysis based on Sign test on all datasets in MTDS scenario.}
	    \label{fig:signtest}
\end{figure*}

Similarly, for all scenarios (i.e. MTDS, ATDS, and BTDS), we carry out the Nemenyi test on the results of TLF with other methods (one by one) at the right-tailed by considering significance level $\alpha=0.025$ (i.e. 97.5\% significance level). While we compare TLF with an existing method (say Q) we obtain a p-value which is used to determine the significant improvement of TLF over Q. The p-value for a comparison (i.e. TLF vs Q) on a dataset is calculated as follows~\cite{demvsar2006statistical}. Let, $n$ and $t$ be the number of source and target pairs of a dataset, and methods, respectively. For a dataset, we have $nt$ accuracies. For each accuracy of a method, we identify the rank among all $nt$ accuracies. So we have $n$ number of ranks for the method. After that we calculate the mean rank of the method. We then use the mean ranks TLF and Q to identify the difference, $d$ and standard error, $s$. We finally calculate the p-value as follows~\cite{demvsar2006statistical}.
{\scriptsize
\begin{equation}
\begin{aligned}
	p=1-f(x)=1- \frac{1}{\sigma\sqrt{2\pi}}e^{-\frac{1}{2}{(\frac{x-\mu}{\sigma})}^2}
\end{aligned}
\label{eq_p_nemenyitest}
\end{equation}
}
where $x=\frac{d}{s}$, mean $\mu=0$ and standard deviation $\sigma=1$. The performance of TLF is significantly better than Q if p is lower than p(ref) which is 0.025 in our case. For MTDS, ATDS, and BTDS scenarios, the Nemenyi test results indicate that TLF performs significantly better than the other techniques (at $p<0.025$) on all datasets. 

\subsection{Effectiveness Verification of TLF Components}
\label{effectivenessverification}
To inspect the effectiveness of each component, in MTDS scenario, we carry out an experimentation on the PIE data as shown in Fig.~\ref{fig:componentanalysis} where the first setting is the TLF with all components, the second setting is the TLF without ridge regularization, the third setting is the TLF without MMD regularization and the fourth setting is the TLF without manifold regularization. From the figure, we can see the importance of all components as the classification accuracy of the first setting (first bar of the figure) is higher than other settings.
\begin{figure}[ht!]
\centering
  \setlength{\belowcaptionskip}{0pt}
	\setlength{\abovecaptionskip}{0pt}	
	\includegraphics[width=0.90\linewidth]{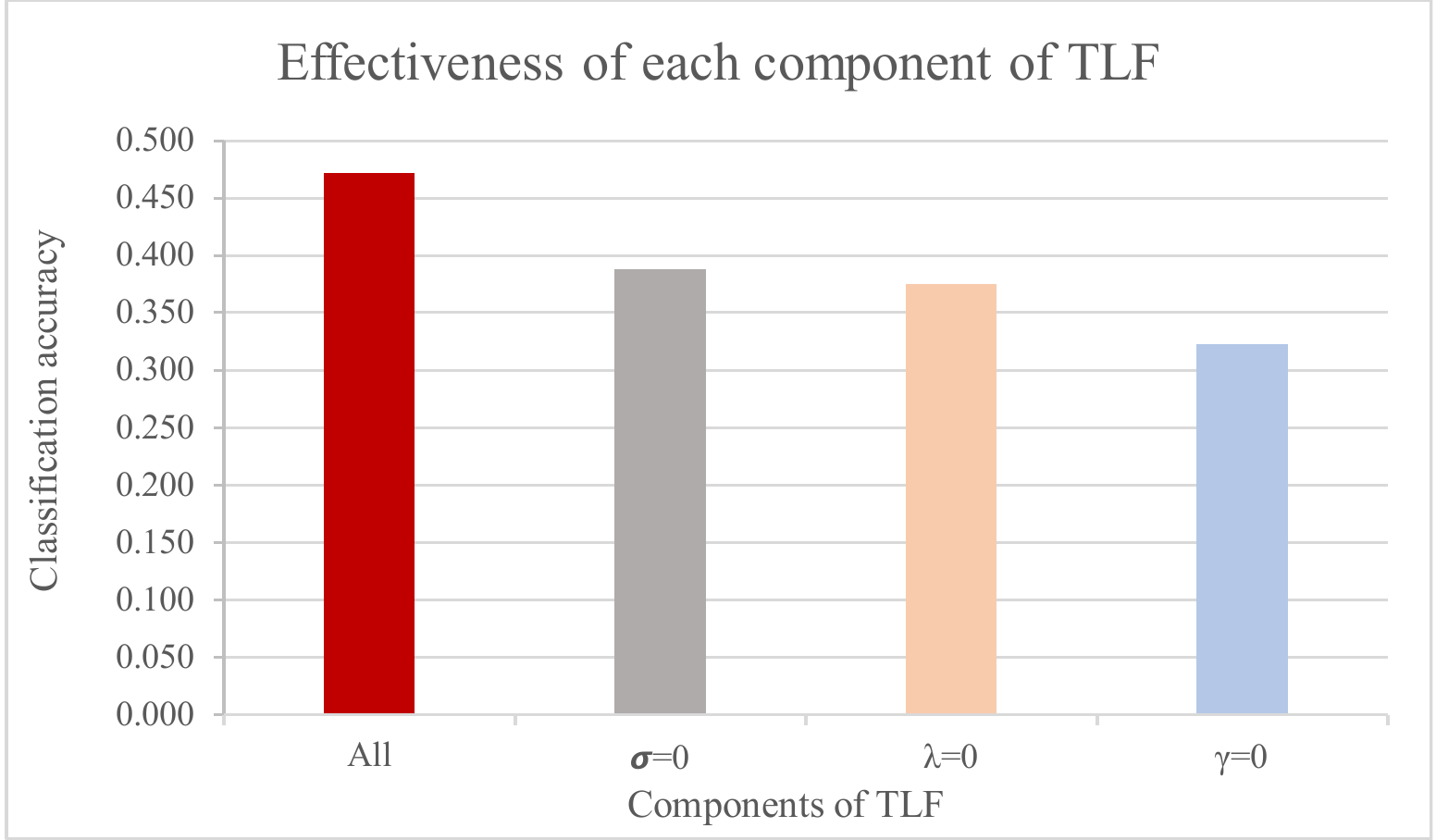}
	 \caption{Effectiveness of the use of ridge, MMD and Manifold regularizations in TLF is demonstrated in terms of classification accuracy on PIE dataset.}
	 \label{fig:componentanalysis}
\end{figure}

\section{Analysis of the Effectiveness of TLF in Some Challenging Situations}
\label{Analysis_Challenging_Situations}

After presenting the main experimental results and evaluation in Section~\ref{experiments}, we now explore the effectiveness of TLF in some challenging situations~\cite{rahman2022adaptive}, such as when datasets have missing values, as follows.

\subsection{Handling Datasets with Missing Values}
\label{missing_values}

Similar to SHDA~\cite{sukhija2019supervised} and ARTL~\cite{long2014adaptation}, in TLF we consider the datasets without missing values. However, in some real applications, due to many reasons datasets may contain missing values~\cite{rahman2013missing}. In RF~\cite{breiman2001random}, missing values are dealt by either (i) using an approach called SRD that simply deletes the records having missing values, or (ii) by imputing missing values with the attribute mean (for numerical) or mode (for categorical) values. If a dataset contains missing values, we suggest to (i) impute the missing values by using a state-of-the-art technique such as SiMI~\cite{rahman2013missing} or DMI\footnote{An implementation of DMI is publicly available in WEKA~\cite{witten2016data}}~\cite{rahman2011decision}, and (ii) consider the imputed dataset as input to TLF.  

We examine the effectiveness of TLF in handling missing values on PIE and OC datasets as shown in Fig.~\ref{fig:tlf_missing_values}. For PIE dataset, we artificially create missing values in all source and target datasets in MTDS scenario. The missing values are created randomly with 5 missing ratios: 10\%, 20\%, 30\%, 40\% and 50\%, where x\% missing ratio means x\% of the total records of a dataset have one or more missing values. The records are chosen randomly. Once a record is selected, y\% of total attribute values of the record are randomly made missing. The value of y is chosen randomly between 1 and 50.

For the SRD approach, we first delete the records having missing values from both source and target datasets and then build a classifier for the target domain by applying TLF on the datasets without missing values. The classifier is then used to classify the test data and calculate the classification accuracy. This approach is called ``SRD-TLF'' in Fig.~\ref{fig:tlf_missing_values}. For an imputation approach, we impute the missing values of both source and target datasets using SiMI~~\cite{rahman2013missing} and then build a classifier by applying TLF on the datasets without missing values. We then use the classifier to classify the test data and calculate the classification accuracy. This approach is called ``SiMI-TLF'' in Fig.~\ref{fig:tlf_missing_values}. The average accuracies of all source and target datasets for SRD-TLF and SiMI-TLF on PIE dataset are presented in Fig.~\ref{fig:pie_mvi} where the average accuracy of ``Original-TLF'' is obtained from the TLF column of Table~\ref{tab:overall_accuracy}. Similarly, for OC dataset, we calculate the average classification accuracies of Original-TLF, SRD-TLF and SiMI-TLF and present them in Fig.~\ref{fig:oc_mvi}. It is clear from the figures that if a dataset contains missing values, TLF performs better if the missing values are imputed instead of deleting the records having missing values. For small missing ratios, SiMI-TLF achieves almost the accuracy of TLF for a dataset without any missing values.

\begin{figure}[ht!]
\centering
  \setlength{\belowcaptionskip}{0pt}
	\setlength{\abovecaptionskip}{0pt}	
	    \subfigure[PIE dataset]
	    {
	        \includegraphics[width=0.45\linewidth]{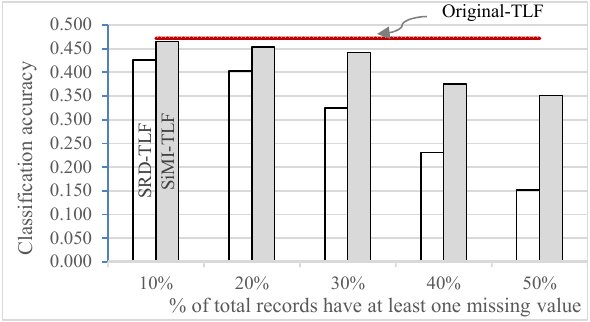}
	        \label{fig:pie_mvi}
	    }
	    \subfigure[OC dataset]
	    {
	        \includegraphics[width=0.45\linewidth]{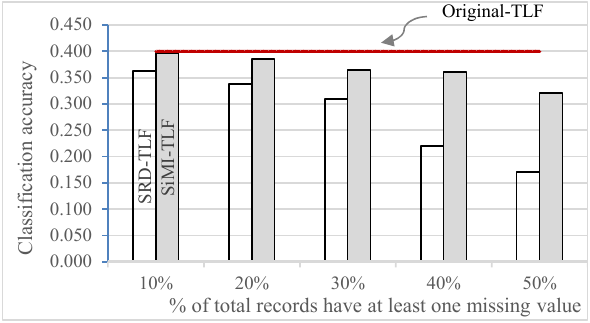}
	        \label{fig:oc_mvi}
	    }
	    \caption{Classification accuracies of Original-TLF, SRD-TLF and SiMI-TLF on PIE and OC datasets.}
	    \label{fig:tlf_missing_values}
\end{figure}

\subsection{Handling Class Imbalanced Datasets}
\label{imbalanced_batches}

Since it is possible to integrate any decision forest algorithm into TLF framework, TLF is capable of handling class imbalanced datasets by integrating an existing algorithm called CSForest~\cite{siers2018novel} which is designated to handle class imbalances problems and we call this variant TLF-C. For implementing the CSForest, we use the Java code from the Weka platform~\cite{witten2016data}. In our experiment, the default cost metrics of CSForest are set to: \textit{true-positive}=1, \textit{true-negative}=0, \textit{false-positive}=1, and \textit{false-negative}=5.

We investigate the effectiveness of TLF-C, TLF, CSForest and 3 existing transfer learning methods, namely SHDA, MEDA and ARTL on two class imbalanced datasets, namely segment0 and yeast1 that are publicly available at Keel repository~\cite{alcala2011keel}. We create source, target and test datasets from each original dataset as follows. We first create source dataset by randomly selecting 70\% records of the original dataset and then create target dataset with the remaining records of the original dataset. After that we divide the target dataset in to test and target datasets by randomly choosing 95\% and 5\%, respectively, records of the actual target dataset. We repeat this process for 10 runs. For each run, we find the result of a method. We calculate the average result of the 10 runs and present it in Table~\ref{tab:overall_avg_accuracy_imbalancebatch}. Since the classification accuracy is not a useful metric for evaluating classifiers built from class imbalanced datasets, we calculate Precision, Recall and F-measure (F1) for each test dataset.

Following the similar process discussed in Section~\ref{experimentalresults}, we obtain results for the methods, and present in Table~\ref{tab:overall_avg_accuracy_imbalancebatch}. Bold values in the table indicate the best results. For the cost metric, TLF-C achieves higher precision, recall and F1 than the other methods. It is clear from table that TLF has the ability to handle class imbalanced datasets by integrating an algorithm such as CSForest~\cite{siers2018novel}.
\begin{table*}[!ht]
\setlength{\belowcaptionskip}{-10pt}
	\setlength{\abovecaptionskip}{-5pt}	
\tiny
	\centering
	\caption{Overall Precision, Recall and F1 of TLF and existing methods on two class imbalanced datasets.}
	\renewcommand\arraystretch{1.2} 
	\renewcommand\tabcolsep{0pt} 
	\begin{tabular*}{\textwidth}{@{\extracolsep{\fill}}lccccccccccccccccccccc}
		\toprule
		\multirow{2}*{Dataset} & \multicolumn{7}{c}{Precision (higher the better)}& \multicolumn{7}{c}{Recall  (higher the better)}& \multicolumn{7}{c}{$F_1$  (higher the better)}\\  
	\cline{2-8}\cline{9-15}\cline{16-22}
		&CSForest-S&CSForest-T&SHDA&MEDA&ARTL&TLF&TLF-C&CSForest-S&CSForest-T&SHDA&MEDA&ARTL&TLF&TLF-C&CSForest-S&CSForest-T&SHDA&MEDA&ARTL&TLF&TLF-C\\
		\midrule
segment0&0.854&0.792&0.715&0.754&0.778&0.812&\textbf{0.914}&0.715&0.701&0.642&0.675&0.688&0.700&\textbf{0.846}&0.778&0.744&0.676&0.712&0.730&0.752&\textbf{0.879}\\
yeast1&0.541&0.481&0.425&0.445&0.451&0.501&\textbf{0.654}&0.488&0.479&0.408&0.412&0.415&0.452&\textbf{0.605}&0.513&0.480&0.416&0.428&0.432&0.475&\textbf{0.629}\\
		\bottomrule
	\end{tabular*}
	\label{tab:overall_avg_accuracy_imbalancebatch}
\end{table*}

\section{Conclusion}
\label{conclusion}

This paper introduced TLF, a supervised heterogeneous transfer learning framework, which learns a classifier for the target domain having only few labeled training records by transferring knowledge from the source domain having many labeled records. Differing form existing methods, in TLF we simultaneously address two key issues: feature discrepancy and distribution divergence. 

To alleviate feature discrepancy, TLF identifies shared label distributions that act as the pivots to bridge the domains (see Step 3 of Algorithm~\ref{algo_tlf} and Fig.~\ref{fig:basicconcept}). To handle distribution divergence, TLF simultaneously optimizes the structural risk functional, joint distributions between domains, and the manifold consistency underlying marginal distributions (see Step 4 of Algorithm~\ref{algo_tlf} and Section~\ref{generalframework}). In TLF, we propose a novel approach to determine nearest neighbors of a record automatically so that intrinsic properties of manifold consistency can be exploited in detail resulting in a high transfer performance (see Section~\ref{generalframework} and Fig.~\ref{fig:justifyautok}). Moreover, to avoid negative transfer, we identify transferable records of the source dataset by using the pivots (see Step 3 of Algorithm~\ref{algo_tlf} and Section~\ref{generalframework}).

The effectiveness of TLF is also reflected in the experimental results. We evaluate TLF on seven publicly available natural datasets and compare the performance of TLF against the performance of fourteen state-of-the-art techniques including SHOT++~\cite{liang2021source}, TNT~\cite{chen2016transfer}, ARTL~\cite{long2014adaptation}, SHDA~\cite{sukhija2019supervised} and MEDA~\cite{wang2020transfer}. We also evaluate how the size of the target dataset affects the classification performance of a transfer learning method and identify three scenarios, namely MTDS, ATDS, and BTDS.    
 From Table~\ref{tab:overall_accuracy}, Table~\ref{tab:overall_accuracy_atds} and Table~\ref{tab:overall_accuracy_btds}, we can see that in all datasets and for all scenarios, TLF outperforms the other techniques in terms of classification accuracy. Statistical sign test and Nemenyi test analyses (see Section~\ref{Experimental_stat_analysis}) indicate that TLF performs significantly better than the other methods at $z>1.96$, $p<0.025$ on all datasets.

From the experiments we highlighted the strengths and weaknesses of TLF as follows. The main strengths are: (1) TLF achieves the best classification performance on all datasets over the eleven state-of-the-art techniques (see Table~\ref{tab:overall_accuracy}, Table~\ref{tab:overall_accuracy_atds} and Table~\ref{tab:overall_accuracy_btds}), (2) TLF is capable of handling both feature discrepancy and distribution divergence issues simultaneously, (3) TLF has the ability to handle imbalanced datasets (see Section~\ref{imbalanced_batches}), and (5) TLF is capable of dealing with datasets having missing values (see Section~\ref{missing_values}). However, like existing methods including ARTL~\cite{long2014adaptation}, SHDA~\cite{sukhija2019supervised} and MEDA~\cite{wang2020transfer}, TLF has a limitation on processing multiple source domains. TLF is now suitable for handling only a single source domain at a time.

In future, we plan to explore the effectiveness of combining records from multiple source domains to learn classifiers for the target domain.

\section*{Acknowledgements}
We give our sincere thanks to the editors and the anonymous reviewers for their time and wise comments that have helped us to improve the quality of the study significantly.
\bibliographystyle{IEEEtran}
\bibliography{gea_pd_ref}

\vspace{-10 mm}
\begin{IEEEbiography}[{\includegraphics[width=1in,height=1.25in,clip,keepaspectratio]{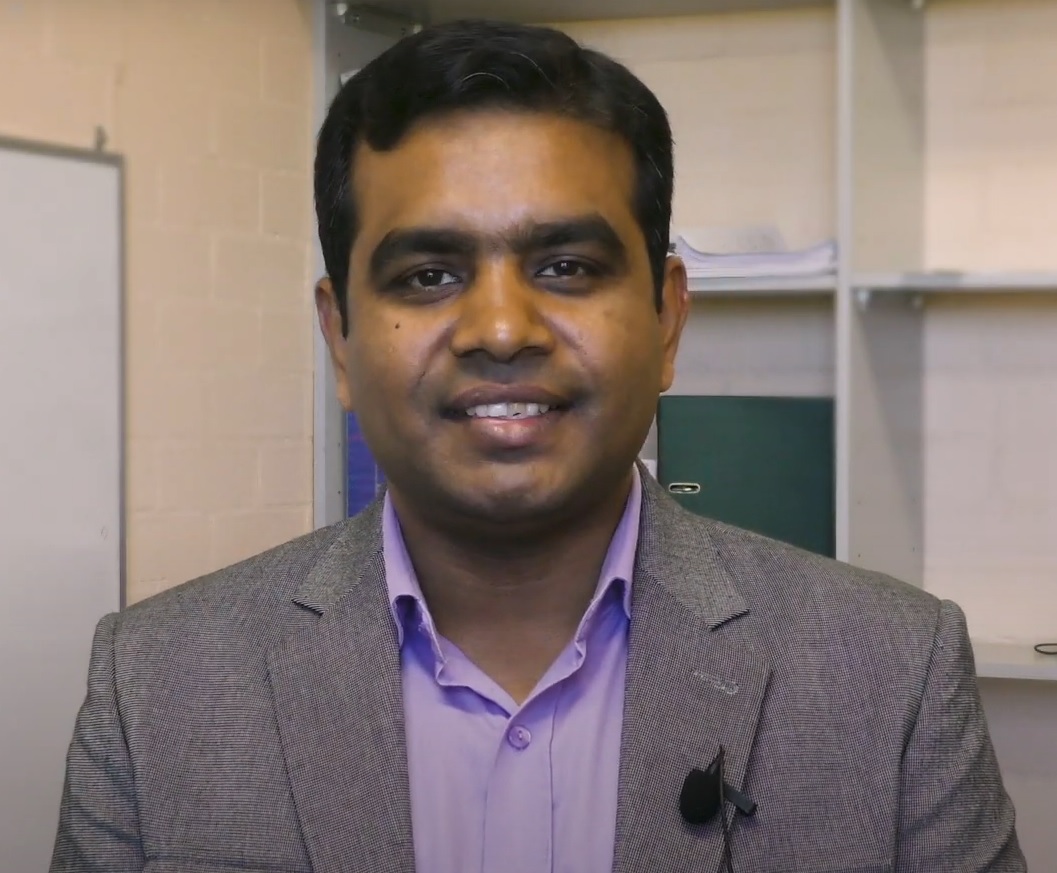}}]{Md Geaur Rahman} is a Research Associate in the School of Computing, Mathematics and Engineering, Charles Sturt University, Australia. He received his PhD degree from the Charles Sturt University, Australia. He is working as a Professor in the Department of Computer Science and Mathematics, Bangladesh Agricultural University, Bangladesh. His research interests include data pre-processing, data mining and knowledge discovery, machine learning, incremental learning and transfer learning. URL: \url{http://gea.bau.edu.bd/}
\end{IEEEbiography}
\vspace{-10 mm}
\begin{IEEEbiography}[{\includegraphics[width=1in,height=1.25in,clip,keepaspectratio]{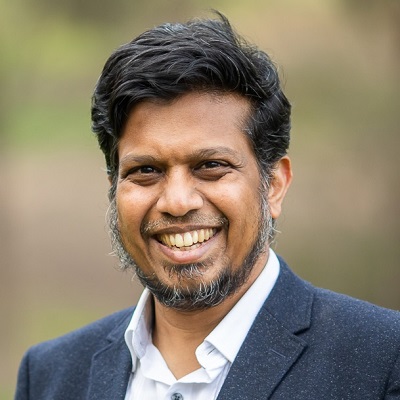}}]{Md Zahidul Islam}
 is a Professor in Computer Science, in the School of Computing, Mathematics and Engineering, Charles Sturt University, Australia. He is serving as the Director of the Data Science Research Unit (DSRU), of the Faculty of Business Justice and Behavioural Sciences, Charles Sturt University, Australia. He received his PhD degree from the University of Newcastle, Australia. His main research interests are in Data Mining, Classification and Clustering algorithms, Missing value analysis, Outliers detection, Data Cleaning and Preprocessing, Privacy Preserving Data Mining, Privacy Issues due to Data Mining on Social Network Users, and Applications of Data Mining in Real Life. URL: \url{http://csusap.csu.edu.au/~zislam/}
\end{IEEEbiography}
\vfill
%

\end{document}